\documentclass{article}
\usepackage[utf8]{inputenc}
\pdfoutput=1
\usepackage[sc,osf]{mathpazo}
\usepackage[margin=1.1in,letterpaper]{geometry}
\usepackage[sort&compress,numbers]{natbib}
\usepackage[colorlinks=true,citecolor=blue,breaklinks]{hyperref}
\usepackage[hyphenbreaks]{breakurl}

\makeatletter
\newlength\aftertitskip     \newlength\beforetitskip
\newlength\interauthorskip  \newlength\aftermaketitskip

\def\maketitle{\par
 \begingroup
   \def\thefootnote{\fnsymbol{footnote}}
   \def\@makefnmark{\hbox to 4pt{$^{\@thefnmark}$\hss}}
   \@maketitle \@thanks
 \endgroup
\setcounter{footnote}{0}
 \let\maketitle\relax \let\@maketitle\relax
 \gdef\@thanks{}\gdef\@author{}\gdef\@title{}\let\thanks\relax}

\def\@startauthor{\noindent \normalsize\bf}
\def\@endauthor{}
\def\@starteditor{\noindent \small {\bf Editor:~}}
\def\@endeditor{\normalsize}
\def\@maketitle{\vbox{\hsize\textwidth
 \linewidth\hsize \vskip \beforetitskip
 {\begin{center} \LARGE\@title \par \end{center}} \vskip \aftertitskip
 {\def\and{\unskip\enspace{\rm and}\enspace}%
  \def\addr{\small\it}%
  \def\email{\hfill\small\tt}%
  \def\name{\normalsize\bf}%
  \def\AND{\@endauthor\rm\hss \vskip \interauthorskip \@startauthor}
  \@startauthor \@author \@endauthor}
}}

\makeatother

\pdfoutput=1                    

\usepackage[utf8]{inputenc} 
\usepackage[T1]{fontenc}    
\usepackage{hyperref}       

\usepackage[usenames,dvipsnames]{xcolor}
\definecolor{darkblue}{rgb}{0,0,0.90}
\hypersetup{
  colorlinks = true,
  citecolor  = darkblue,
  linkcolor  = darkblue,
  citecolor  = darkblue,
  filecolor  = darkblue,
  urlcolor   = darkblue,
}

\usepackage{algorithm,algorithmic}
\usepackage{graphicx}
\usepackage{bbold}
\usepackage{framed}
\usepackage{amssymb}
\usepackage{amsfonts}
\usepackage{mathrsfs}
\usepackage{mathtools}
\usepackage{array}
\usepackage{amsthm}
\usepackage{verbatim} 
\usepackage{enumerate}
\usepackage{bbm}
\usepackage{mathtools}
\usepackage{commath}
\usepackage{wrapfig}
\usepackage{amsbsy}
\usepackage[sc,osf]{mathpazo}
\usepackage{amsmath}
\usepackage{changepage}
\usepackage{nicefrac}

\usepackage[font=small,labelfont=bf,tableposition=top]{caption}
\usepackage{cuted}

\newtheorem{prop}{Proposition}
\newtheorem{lemma}[prop]{Lemma}

\newtheorem{theorem}[prop]{Theorem}

\theoremstyle{definition}
\newtheorem{defn}[prop]{Definition}

\newcommand{\sj}[1]{\textcolor{red}{SJ: #1}}

\usepackage{chngcntr}
\usepackage{apptools}
\AtAppendix{\counterwithin{prop}{section}}

\usepackage{microtype}
\usepackage{graphicx}
\usepackage{subfigure}
\usepackage{booktabs} 

\usepackage{hyperref}

\title{Strength from Weakness: Fast Learning Using Weak Supervision}

\author{\name Joshua Robinson \email{joshrob@mit.edu}\\
  \name Stefanie Jegelka \email{stefje@csail.mit.edu}\\
  \name Suvrit Sra \email{suvrit@mit.edu}\\
  \addr{Massachusetts Institute of Technology, Cambridge, MA 02139} 
}

\begin{document}

\maketitle

\begin{abstract} 
  We study generalization properties of weakly supervised learning. That is, learning where only a few ``strong'' labels (the actual target of our prediction) are present but many more ``weak'' labels are available. In particular, we show that having access to weak labels can significantly accelerate the learning rate for the strong task to the fast rate of $\mathcal{O}(\nicefrac1n)$, where $n$ denotes the number of strongly labeled data points. This acceleration can happen even if by itself the strongly labeled data admits only the slower  $\mathcal{O}(\nicefrac{1}{\sqrt{n}})$ rate. The actual acceleration depends continuously on the number of weak labels available, and on the relation between the two tasks. Our theoretical results are reflected empirically across a range of tasks and illustrate how weak labels speed up learning on the strong task.
\end{abstract}

\vspace{-15pt}
\section{Introduction}
While access to large amounts of labeled data has enabled the training of big models with great successes in applied machine learning, it remains a key bottleneck. In numerous settings (e.g., scientific measurements, experiments, medicine) obtaining a large number of labels can be prohibitively expensive, error prone, or otherwise infeasible. 
When labels are scarce, a common alternative is to use additional sources of information: ``weak labels'' that contain information about the ``strong'' target task and are more readily available, e.g., a related task, or noisy versions of strong labels from non-experts or cheaper measurements. \\

Such a setting is called \emph{weakly supervised learning}, and given its great practical relevance it has received much attention \cite{zhou2018brief,pan2009survey,liao2005logistic,dai2007boosting,huh2016makes}. A prominent example that enabled breakthrough results in computer vision and is now standard, is to pre-train a complex model on a related, large data task, and to then use the learned features for fine-tuning for instance the last layer on the small-data target task~\citep{girshick14,donahue14,zeiler14,sun17}. Numerous approaches to weakly supervised learning have succeeded in a variety of tasks; beyond computer vision \cite{oquab2015object,durand2017wildcat,carreira17,fries2019weakly}. Examples include clinical text classification~\cite{wang2019clinical}, 
sentiment analysis \cite{medlock2007weakly}, social media content tagging \cite{mahajan2018exploring} and many others. Weak supervision is also closely related to unsupervised learning methods such as complementary and contrastive learning \cite{xu2019generative,chen2019weakly,arora2019theoretical}, and particularly to self-supervised learning \cite{doersch2015unsupervised}, where feature maps learned via supervised training on artificially constructed tasks have been found to even outperform ImageNet learned features on certain downstream tasks \cite{misra2019self}. \\

In this paper, we make progress towards building theoretical foundations for weakly supervised learning, i.e., where we have a few strong labels, but too few to learn a good model in a conventional supervised manner. 
Specifically we ask, \\
\begin{adjustwidth}{15pt}{15pt}
\emph{Can large amounts of weakly labeled data provably help learn a better model than strong labels alone?}\\
\end{adjustwidth}

We answer this question positively by analyzing a generic feature learning algorithm that learns features on the weak task, and uses those features in the strong downstream task.
While generalization bounds for supervised learning typically scale as $\mathcal{O}(\nicefrac{1}{\sqrt{n}})$, where $n$ is the number of strongly labeled data points, we show that the feature transfer algorithm can do better, achieving the superior rate of  $\widetilde{ \mathcal{O} }(n^{-\gamma})$ for $1/2 \leq \gamma \leq 1$, where $\gamma$ depends on how much weak data is available, and on generalization error for the weak task. This rate smoothly interpolates between  $\widetilde{\mathcal{O}}(\nicefrac 1 n)$ in the best case, when weak data is plentiful and the weak task is not too difficult, and slower rates when less weak data is available or the weak task itself is hard.
One instantiation of our results for categorical weak labels says that, if we can train a model with $\mathcal{O}(\nicefrac{1}{\sqrt{m}})$ excess risk for the weak task (where $m$ is the amount of weak data), and $m = \Omega(n^2)$, then we obtain a ``fast rate'' $\widetilde{\mathcal{O}}(\nicefrac 1 n)$ on the excess risk of the \emph{strong task}. This speedup is significant compared to the commonly observed  $\mathcal{O}(\nicefrac{1}{\sqrt{n}})$ ``slow rates''. \\

In order to obtain any such results, it is necessary to capture the task relatedness between weak and strong tasks. We formalize and quantify this relatedness via the \emph{central condition} \citep{erven2012mixability, van2015fast}. This condition essentially implies that there is a suitable feature embedding that works well on the weak task, and also helps with the strong task. The challenge, however, arises from the fact that we do not know \emph{a priori} what the suitable embedding is. The main part of our theoretical work is devoted to establishing that using instead an embedding learned from weak data still allows fast learning on the strong task. \\

%

In short, we make the following contributions:
\begin{itemize}\setlength{\itemsep}{-1pt}
\item We introduce a theoretical framework for analyzing weakly supervised learning problems.
  
\item We propose the central condition as a viable way to quantify relatedness between weak and strong tasks. The condition requires there is an embedding that is good for both tasks, but which is unobservable; this makes obtaining generalization bounds non-trivial.
\item We obtain generalization bounds for the strong task. These bounds depend  continuously on two key quantities: 1) the growth rate of the number $m$ of weak labels in terms of the number $n$ of strong labels, and 2)~generalization performance on the weak task.
\item We show that in the best case, when $m$ is sufficiently larger than $n$, weak supervision delivers \emph{fast rates}.
  \vspace*{4pt}
\end{itemize}
We validate our theoretical findings, and observe that our fast and intermediate rates are indeed borne out in practice.
 \vspace{-5pt}

\subsection{Examples of Weak Supervision}
\noindent
\textbf{Coarse Labels.}  
It is often easier to collect labels that capture only part of the information about the true label of interest \cite{zhao2011large,guo2018cnn,yan2015hd, taherkhani2019weakly}. A particularly pertinent example is semantic labels obtained from hashtags attached to images \cite{mahajan2018exploring,li2017webvision}. Such tags are generally easy to gather in large quantities, but tend to only capture certain aspects of the image that the person tagging them focused on. For example, an image with the tag \textbf{\texttt{\#dog}} could easily also contain children, or other  label categories that have not been explicitly tagged. \\

\noindent
\textbf{Crowd Sourced Labels.}  A primary way for obtaining large labeled data is via crowd-sourcing using platforms such as Amazon Mechanical Turk \cite{khetan2017learning, kleindessner2018crowdsourcing}. Even for the simplest of labeling tasks, crowd-sourced labels can often noisy \cite{zhang2018generalized, branson2017lean,zhang2014spectral}, which becomes worse for labels requiring expert knowledge.
Typically, more knowledgeable labelers are more expensive (e.g., professional doctors versus medical students for a medical imaging task), which introduces a tradeoff between label quality and cost that the user must carefully manage. \\

\noindent
\textbf{Object Detection.} A common computer vision task is to draw bounding boxes around objects in an image \cite{oquab2015object}. A popular alternative to expensive bounding box annotations is a collection of words describing the objects present, without localization information~\cite{bilen2016weakly,branson2017lean,wan2018min}. This setting is also an instance of coarse labeling. \\


\noindent
\textbf{Model Personalization.} In examples like recommender systems \cite{ricci2011introduction}, online advertising \cite{naumov2019deep}, and personalized medicine \cite{schork2015personalized}, one needs to make predictions for individuals, while using information shared by a larger population as supportive, weak supervision \cite{desrosiers2011comprehensive}.




\vspace*{-6pt}
\section{Weakly Supervised Learning}

We begin with some notation. The spaces $\cal X$ and $\cal Y$ denote as usual the space of features and strong labels. In \emph{weakly supervised learning}, we have in addition $\cal W$, the space of weak labels. We receive the tuple $(X,W,Y)$ drawn from the product space $\mathcal{X} \times \mathcal{W} \times \mathcal{Y}$. The goal is to then predict the strong label $Y$ using the features $X$, and  possibly benefiting from the related information captured by $W$.

More specifically, we work with two datasets: (1) a weakly labeled dataset $\mathcal{D}_m^\text{weak}$ of $m$ examples drawn independently from the marginal distribution $P_{X,W}$; and (2) a dataset $\mathcal{D}_n^\text{strong}$ of $n$ strong labeled examples drawn from the marginal $P_{X,Y}$. Typically, $n \ll m$.
We then use the weak labels to learn an embedding in a latent space $\mathcal{Z} \subset \mathbb{R}^s$. In particular, we assume that there exists an unknown ``good'' embedding $Z = g_0(X) \in \mathcal{Z}$, using which a linear predictor $\beta_{g_0}$ can determine $W$, i.e., $\beta_{g_0}^\top Z =\beta_{g_0}^\top  g_0(X) = W$. The strong equality assumption can be relaxed via an additive error term in our risk bounds that capture the risk of $\beta_{g_0}^\top  g_0$. \\

Using the latent space $\cal Z$, we define two function classes: \emph{strong predictors} $\mathcal{F} \subset \{ f : \mathcal{X} \times \mathcal{Z} \rightarrow \mathcal{Y} \}$, and \emph{weak feature maps} $\mathcal{G} \subset \{ g : \mathcal{X} \rightarrow \mathcal{Z} \}$.  Later we will assume that class $\mathcal{F}$ is parameterized, and identify functions $f$ in $\mathcal{F}$ with parameter vectors. We then learn a predictor $f \in \mathcal{F}$ by replacing the latent vector $Z$ with an embedding $\hat{g}(X) \in \mathcal{Z}$ that we learn from weakly labeled data. Corresponding to these function classes we introduce two loss functions. \\

First, $\ell : \mathcal{Y} \times \mathcal{Y} \rightarrow \mathbb{R}_+$ measures loss of the strong predictor; we assume this loss to be continuously differentiable in its first argument. We will equivalently write 
$\ell_{f} (x,z,y) := \ell (f(x,z),y)$ for predicting from a latent vector $z \in \mathcal{Z}$; similarly, for predicting from an estimate $\hat{z} = g(x)$, we write the loss as $\ell_{f(\cdot , g)} (x,y) := \ell (f(x,g(x)),y)$. \\

Second, $\ell^\text{weak} : \mathcal{W} \times \mathcal{W} \rightarrow \mathbb{R}_+$ measures loss for the weak task. This loss also applies to measuring loss of feature maps $g : \mathcal{X} \rightarrow \mathcal{Z}$, by using the best possible downstream linear classifier, i.e., $ \ell^\text{weak}_g(x,w) = \ell^\text{weak}(\beta_g ^\top g(x),w)$ where $ \beta_g \in \arg \min_{\beta \in \mathbb{R}^s}  \mathbb{E} \ell^{\text{weak}} (\beta ^\top g(X) , W)$. Our primary goal is to learn a model $\hat{h} = \hat{f}( \cdot , \hat{g} ): \mathcal{X} \rightarrow \mathcal{Y}$ that achieves low risk  $\mathbb{E}  [\ell_{\hat{h}}(X,Y)  ]$.
To that end, we seek to bound the \emph{excess risk}:
\begin{equation}
\label{eq:1}
\mathbb{E}_{P} [\ell_{\hat{h}}(X,Y) - \ell_{h^*}(X,Y) ],
\end{equation}
for $h^* = f^*( \cdot , g^*)$ where $g^*$ and $f^*$ are given by
\begin{align*}
  g^* &\in \mathrm{argmin}_{g \in \mathcal{G}}  \mathbb{E} [\ell^\text{weak}_g(X,W)],\\
  f^* &\in \mathrm{argmin}_{f \in \mathcal{F}}  \mathbb{E} [\ell_{f(\cdot, g^*)}(X,Y)].
\end{align*}
The comparison of $\hat{h}$ to $h^*$ based on the best weak task model $g^*$ is the most natural one for the feature transfer algorithm that we analyze (Algorithm \ref{algo:double-estimation}). We study the \emph{rate} at which the excess risk~\eqref{eq:1} goes to zero. Specifically, if the excess risk is $\mathcal{O}(n^{-\gamma})$, the learning rate is $\gamma$. We refer to $\gamma \leq 1/2$ as a \emph{slow rate}, and $\gamma \geq 1$ as a \emph{fast rate} (possibly ignoring logarithmic factors, i.e., $\widetilde{\mathcal{O}}(\nicefrac 1 n)$). When $1/2 < \gamma < 1$ we have \emph{intermediate rates}.

  \begin{figure}[t]
  \centering
  \includegraphics[width=55mm]{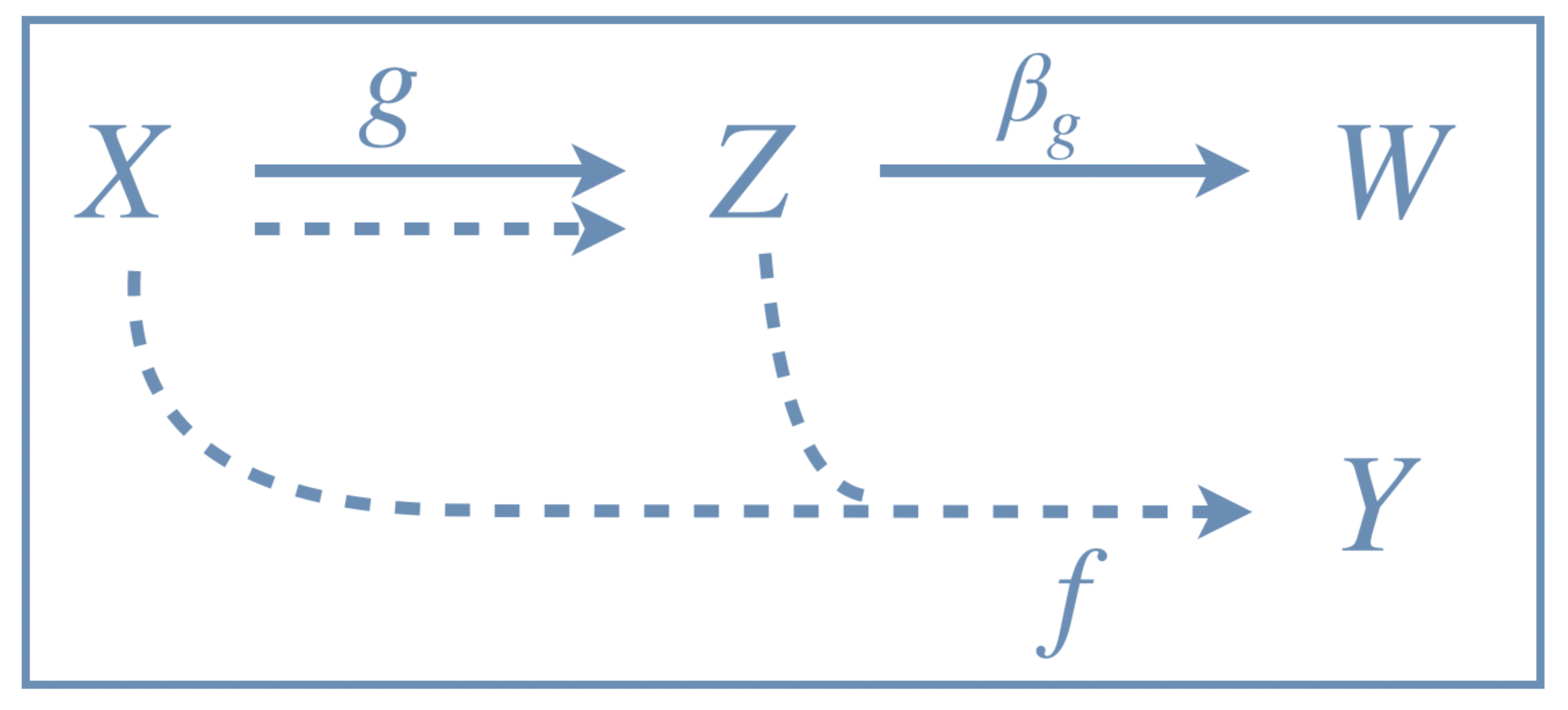}
  \caption{Schema for weakly supervised learning using Algorithm \ref{algo:double-estimation}. The dotted lines denote the flow of strong data, and the solid lines the flow of weak data.}
    \label{fig: schema}
 \end{figure}
 \vspace{-9pt}

\vspace{+3pt}
\subsection{Feature transfer meta-algorithm}

The algorithm we analyze solves two supervised learning problems in sequence. The first step runs an algorithm, 

\[ \hat{g} \leftarrow \text{Alg}_m(\mathcal{G}, P_{X,W})\]

on $m$ i.i.d.\ observations from $P_{X,W}$, and outputs a feature map $\hat{g} \in \mathcal{G}$. Using the resulting $\hat{g}$ we form an augmented dataset $\mathcal{D}_n^\text{aug} = \{(x_i, z_i, y_i) \}_{i=1}^n$, where $z_i := \hat{g}(x_i)$ for $(x_i,y_i) \in \mathcal{D}_n^\text{strong}$. Therewith, we have $n$ i.i.d.\ samples from the distribution $\hat{P}(X,Z,Y) := P(X,Y) \mathbb{1}\{ Z=\hat{g}(X)\}$. The second step then runs an algorithm,

\[ \hat{f} \leftarrow \text{Alg}_n(\mathcal{F}, \hat{P})\]

on $n$ i.i.d samples from $\hat{P}$, and outputs a  strong predictor $\hat{f} \in \mathcal{F}$. The final output is then simply the composition  $\hat{h} = \hat{f}( \cdot , \hat{g} )$. This procedure is summarized in Algorithm~\ref{algo:double-estimation} and the high level schema in Figure \ref{fig: schema}. \\

\begin{algorithm}[t]
    \caption{Feature transfer meta-algorithm}\small 
    \label{algo:double-estimation}
    \begin{algorithmic}[1] 
	    \STATE \textbf{input} $ \mathcal{D}_m^\text{weak}$, $ \mathcal{D}_n^\text{strong}$, $\mathcal{F}$, $\mathcal{G}$
	    \STATE Obtain weak predictor $\hat{g} \leftarrow \text{Alg}_m(\mathcal{G}, P_{X,W})$
	    \STATE  Form dataset $\mathcal{D}_n^\text{aug} = \{(x_i, z_i, y_i) \}_{i=1}^n$ where $z_i := \hat{g}(x_i)$ for $(x_i,y_i) \in \mathcal{D}_n^\text{strong}$
	    \STATE Define distribution $\hat{P}(X,Z,Y) = P(X,Y) \mathbb{1}\{ Z=\hat{g}(X)\}$
	    \STATE Obtain strong predictor $\hat{f} \leftarrow \text{Alg}_n(\mathcal{F}, \hat{P})$
            \STATE \textbf{return} $\hat{h}(\cdot) := \hat{f}( \cdot , \hat{g}(\cdot) )$
    \end{algorithmic}
\end{algorithm}
\vspace{-3pt}

Algorithm~\ref{algo:double-estimation} is generic because in general the two supervised learning steps can use any learning algorithm. Our analysis treats the case where $\text{Alg}_n(\mathcal{F}, \hat{P})$ is empirical risk minimization (ERM) but is \emph{agnostic} to the choice of learning algorithm $\text{Alg}_m(\mathcal{G}, P_{X,W})$. Our results use high level properties of these two steps, in particular their generalization error, which we introduce next. We break the generalization analysis into two terms depending on the bounds for each of the two supervised learning steps. We introduce here the notation $\text{Rate}(\cdot)$ to enable a more convenient discussion of these rates. 
We describe our notation in the format of definitions to expedite the statement of the theoretical results in Section \ref{sec: analysis}.

	
	
 

\begin{defn}[Weak learning]
  \label{assumption: rates regression}
  Let $ \text{Rate}_m(\mathcal{G}, P_{X,W} ; \delta)$ be such that a (possibly randomized) algorithm $\text{Alg}_m(\mathcal{G}, P_{X,W})$ that takes as input a function class $\mathcal{G}$ and $m$ i.i.d.\ observations from $P_{X,W}$, returns a weak predictor $\hat{g} \in \mathcal{G}$ for which,\\[-8pt]
  \[   \mathbb{E}_P \ell^{\text{weak}} _{\hat{g}}(  X , W) \leq \text{Rate}_m(\mathcal{G}, P_{X,W} ; \delta), \]
  with probability at least $1- \delta$.
\end{defn} 

We are interested in two particular cases of loss function $\ell^\text{weak}$: (i) $\ell^\text{weak}(w,w') = \mathbb{1}\{ w \neq w'\}$ when $\mathcal{W}$ is a categorical space; and (ii) $\ell^\text{weak}(w,w') = \norm{w-w'}$ (for some norm $\norm{\cdot}$ on $\mathcal{W}$) when $\mathcal{W}$ is a continuous space.

\begin{defn}[Strong learning]
  \label{assumption: step two rates}
  Let $ \text{Rate}_n(\mathcal{F}, Q; \delta)$ be such that a (possibly randomized) algorithm $\text{Alg}_n(\mathcal{F}, Q)$ that takes as input a function space $\mathcal{F} $, and $n$ i.i.d.\ observations from a distribution $Q(\mathcal{X} \times \mathcal{Z} \times \mathcal{Y})$, returns a strong predictor $\hat{f} \in \mathcal{F}$ for which,
  \vspace{-3pt}
  \[  \mathbb{E} _{U \sim Q }\big [\ell_{\hat{f}}( U) -\ell_{f^*}(U)\big  ]   \leq \text{Rate}_n(\mathcal{F}, Q; \delta) \]
  \vspace{-3pt}
  with probability at least $1- \delta$.
\end{defn} 
Henceforth, we drop $\delta$ from the rate symbols, for example writing $\text{Rate}_m(\mathcal{G}, P_{X,W})$ instead of $\text{Rate}_m(\mathcal{G}, P_{X,W} ; \delta)$. It is important to note that the algorithms $\text{Alg}_m(\mathcal{G}, P_{X,W})$  and $\text{Alg}_n(\mathcal{F}, Q)$ can use any loss functions during training. This is because the only requirement we place is that they imply generalization bounds in terms of the losses $\ell^\text{weak}$ and $\ell$ respectively. For concreteness, our analysis focuses the case where  $\text{Alg}_n(\mathcal{F}, Q)$ is ERM using loss $\ell$. 



\section{Excess Risk Analysis}\label{sec: analysis}
In this section we analyze Algorithm \ref{algo:double-estimation} with the objective of obtaining high probability excess risk bounds (see~\eqref{eq:1}) for the strong predictor $\hat{h} = \hat{f}(\cdot , \hat{g})$. Informally, the main theorem we prove is the following.

\begin{theorem}[Informal]\label{thm:informal}
Suppose that $\text{Rate}_m(\mathcal{G}, P_{X,W} ) = \mathcal{O}(m^{-\alpha})$ and that $\text{Alg}_n(\mathcal{F},\hat{P})$ is ERM. Under suitable assumptions on $(\ell, P, \mathcal{F})$, Algorithm \ref{algo:double-estimation} obtains excess risk,
\begin{equation*}
\mathcal{O} \bigl(  \frac{ \alpha \beta \log n + \log(1/\delta)}{n} + \frac{ 1 }{n^{\alpha\beta}}   \bigr)    
\end{equation*}
with probability $1- \delta$, when $m = \Omega(n^{\beta})$ for $\mathcal{W}$ discrete, or $m = \Omega(n^{2\beta})$ for $\mathcal{W}$ continuous.
\end{theorem}

For the prototypical scenario where $\text{Alg}_m(\mathcal{G}, P_{X,W} ) = \mathcal{O}(\nicefrac{1}{\sqrt{m}})$, one obtains fast rates when $m = \Omega(n^{2})$, and $m = \Omega(n^{4})$, in the discrete and continuous cases, respectively. More generally, if $\alpha \beta < 1$ then $\mathcal{O}(n^{-\alpha \beta})$ is the dominant term and we observe intermediate or slow rates. \\

In order to obtain any such result, it is necessary to quantify how the weak and strong tasks relate to one another -- if they are completely unrelated, then there is no reason to expect the representation $\hat{g}(X)$ to benefit the strong task. The next subsection introduces the \emph{central condition} and a relative Lipschitz property, which embody the assumptions used for relating the weak and strong tasks. Roughly, they ask that $g_0(X)$ is a useful representation for the strong task.



\subsection{Relating weak and strong tasks}
In this section we introduce the central condition and our relative Lipschitz assumption for quantifying task relatedness. The Lipschitz property requires that small perturbations to the feature map $g$ that do not hurt the weak task, do not affect the strong prediction loss much either. 

\begin{defn}
We say that $f$ is $L$-Lipschitz relative to $\mathcal{G}$ if for all $x \in \mathcal{X} $, $y \in \mathcal{Y} $, and $g,g' \in \mathcal{G}$,\\
\begin{equation*}
  |\ell_{f(\cdot, g)}( x, y) - \ell_{f(\cdot, g)}( x, y) | \leq L \ell^\text{weak}(\beta_g ^\top g(x),
  \beta_{g'}^\top g'(x))).
\end{equation*} \\
We say the function class $\mathcal{F}$ is $L$-Lipschitz relative to $\mathcal{G}$, if every $f \in \mathcal{F}$ is $L$-Lipschitz relative to $\mathcal{G}$. 
\end{defn}

This Lipschitz terminology is justified since the domain uses the pushforward pseudometric $(z,z') \mapsto \ell^\text{weak}(\beta_g ^\top z, \beta_{g'}^\top z')$, and the range is a subset of $\mathbb{R}_+$. In the special case where $\mathcal{Z} = \mathcal{W}$, and $g(X)$ is actually an estimate of the weak label $W$, our Lipschitz condition reduces to $|\ell_{f(\cdot, g)}( x, y) - \ell_{f(\cdot, g)}( x, y) | \leq L \ell^\text{weak}(g(x),g'(x))$, i.e., conventional Lipschitzness of $\ell (f(x,w),y)$ in $w$. \\

The central condition is well-known to yield fast rates for supervised learning~\cite{van2015fast}; it  directly implies that we could learn a map $(X,Z) \mapsto Y$ with $\widetilde{\mathcal{O}}(1/n)$ excess risk. The difficulty with this naive view is that at test time we would need access to the latent value $Z = g_0(X)$, an implausible requirement. To circumnavigate this hurdle, we replace $g_0$ with $\hat{g}$ by solving the supervised problem $(\ell, \hat{P}, \mathcal{F})$, for which we will have access to data.\\

But it is not clear whether this surrogate problem would continue to satisfy the central condition. One of our main theoretical contributions is to show that $(\ell, \hat{P}, \mathcal{F})$ indeed satisfies a weak central condition (Theorems \ref{prop: categorical strong implies epsilon weak} and \ref{prop: regression strong implies epsilon weak}), and to show that this weak central condition still enables strong excess risk guarantees (Theorem~\ref{prop: generalization for weak central condition}). We are now ready to define the central condition. In essence, this condition requires that $(X,Z)$ is highly predictive of $Y$, which, combined with the fact that $g_0(X)=Z$ has zero risk on $W$ links the weak and strong tasks together. 

\begin{defn}[The Central Condition]\label{def: central condition}
A learning problem $(\ell, P, \mathcal{F} )$ on $\mathcal{U} := \mathcal{X} \times \mathcal{Z} \times \mathcal{Y}$ is said to satisfy the \emph{$\varepsilon$-weak $\eta$-central condition} if there exists an $f^* \in \mathcal{F}$ such that\\

\begin{equation*}
  \mathbb{E}_{U\sim P(\mathcal{U})}[ e^{-\eta ( \ell_{f}(U) - \ell_{f^*}(U))}] \leq e^{ \eta \varepsilon},
\end{equation*} \\
for all $f \in \mathcal{F}$. The $0$-weak central condition is known as the strong central condition. 
\end{defn}
We drop the $\eta$ notation when it is being viewed as constant. For the strong central condition, Jensen's inequality implies that $f^*$ must satisfy $\mathbb{E}_P [\ell_{f^*}(U)] \leq \mathbb{E}_P[ \ell_{f}(U)]$ for all $f \in \mathcal{F}$.  The strong central condition is therefore a stronger requirement than the assumption that $\inf_{f \in \mathcal{F} }\mathbb{E}_P[ \ell_{f}(U)]$ is attained. Note that the weak central condition becomes stronger as $\varepsilon$ decreases. Later we derive generalization bounds that improve accordingly as $\varepsilon$ decreases. Before continuing, we take a digression to summarize the central condition's connections to other theory of fast rates.

\vspace*{-6pt}
\paragraph{The central condition and related conditions.}
The central condition unifies many well-studied conditions known to imply fast rates~\citep{van2015fast}, including Vapnik and Chervonenkis' original condition, that there is an $f^* \in \mathcal{F}$ with zero risk \cite{vapnik:264, vapnik1974theory}. The popular strong-convexity condition \cite{kakade2009generalization, lecue2014optimal} is also a special case, as is (stochastic) exponential concavity, which is satisfied by density estimation: where $\mathcal{F}$ are probability densities, and $\ell_f(u) = - \log f(u)$ is the logarithmic loss \cite{audibert2009fast, juditsky2008learning, dalalyan2012mirror}. 
Another example is Vovk mixability \cite{vovk1990aggregating, vovk1998game}, which holds for online logistic regression \cite{ foster2018logistic}, and also holds for uniformly bounded functions with the square loss. A modified version of the central condition also generalizes the Bernstein condition and Tysbakov's margin condition \cite{bartlett2006empirical, tsybakov2004optimal}.


\vspace*{-6pt}
\paragraph{Capturing task relatedness with the central condition.} 
Intuitively, the strong central condition requires that the minimal risk model $f^*$ attains a higher loss than $f \in \mathcal{F}$ on a set of $U = (X,Z,Y)$ with exponentially small probability mass.  This is likely to happen when $(X,Z)$ is highly predictive of $Y$ so that the probability mass of $P(Y|X,Z)$ concentrates in a single location for most $(X,Z)$ pairs. In other words, $(X,Z)$ is highly predictive of $Y$. Further, if $f^*$ in $\mathcal{F}$ such that $f^*(X,Z)$ maps into this concentration, then $\ell_{f^*}(U)$ will be close to zero most of the time, making it probable that Definition \ref{def: central condition}  holds. \\



We also assume that the strong central condition holds for the learning problem $(\ell, P, \mathcal{F} )$ with $P = P_U= P_{X,Z,Y}$ where $Z = g_0(X)$. 
But as noted earlier, since $Z$ is not observable at test time, we cannot simply treat the problem as a single supervised learning problem. Therefore, obtaining fast or intermediate rates is a nontrivial challenge. We approach this challenge by splitting the learning procedure into two supervised tasks (Algorithm \ref{algo:double-estimation}). In its second step, Algorithm~\ref{algo:double-estimation} replaces $(\ell, P, \mathcal{F} )$ with  $(\ell, \hat{P}, \mathcal{F})$. Our strategy to obtain generalization bounds is first to guarantee that $(\ell,  \hat{P}, \mathcal{F})$ satisfies the weak central condition, and then show that the weak central condition implies the desired generalization guarantees. \\

The rest of this section develops the theoretical machinery needed for obtaining our bounds. We summarize the key steps of our argument below.
\begin{enumerate}
  \setlength{\itemsep}{-1pt}
\item Decompose the excess risk into two components: the excess risk of the weak predictor and the excess risk on the learning problem $(\ell, \hat{P},\mathcal{F})$ (Proposition \ref{thm: general bound}).
\item Show that the learning problem $(\ell,  \hat{P}, \mathcal{F})$ satisfies a relaxed version of the central condition - the ``weak central condition'' (Propositions \ref{prop: categorical strong implies epsilon weak} and \ref{prop: regression strong implies epsilon weak}).
\item Show that the $\varepsilon$-weak central condition yields excess risk bounds that improve as $\varepsilon$ decreases (Prop.~\ref{prop: generalization for weak central condition}).
\item Combine all previous results to obtain generalization bounds for Algorithm  \ref{algo:double-estimation} (Theorem~\ref{thm: fast rates}).
\end{enumerate}

\subsection{Generalization Bounds for Weakly Supervised Learning}
The first item on the agenda is Proposition \ref{thm: general bound} which obtains a generic bound on the excess risk in terms of $\text{Rate}_m(\mathcal{G}, P_{X,W} ) $ and $\text{Rate}_n(\mathcal{F}, \hat{P})$. 
\begin{prop}[Excess risk decomposition]\label{thm: general bound}
Suppose that $f^*$ is $L$-Lipschitz relative to $\mathcal{G}$. Then the excess risk  $\mathbb{E}  [\ell_{\hat{h}}(X,Y) - \ell_{h^*}(X,Y) ]$ is bounded by,
\[ 2L \text{Rate}_m(\mathcal{G}, P_{X,W} ) +  \text{Rate}_n(\mathcal{F}, \hat{P}).  \]
\end{prop}
The first term corresponds to excess risk on the weak task, which we expect to be small since that environment is data-rich. Hence, the problem of obtaining excess risk bounds reduces to bounding the second term, $\text{Rate}_n(\mathcal{F}, \hat{P})$.  This second term is much more opaque; we spend the rest of the section primarily analyzing it. \\


We now prove that if $(\ell, P, \mathcal{F} )$ satisfies the $\varepsilon$-weak central condition, then the artificial learning problem $(\ell, \hat{P}, \mathcal{F} )$ obtained by replacing the true population distribution $P$ with the estimate $\hat{P}$ satisfies a slightly weaker central condition. We consider the categorical and continuous $\mathcal{W}$-space cases separately, obtaining an improved rate in the categorical case. In both cases, the proximity of this weaker central condition to the  $\varepsilon$-weak central condition is governed by $\text{Rate}_m(\mathcal{G}, P_{X,W})$, but the dependencies are different. \\
\vspace{-3pt}

\begin{prop}[Categorical weak label]\label{prop: categorical strong implies epsilon weak}
Suppose that $\ell^\text{weak}(w,w') = \mathbb{1}\{ w \neq w'\}$ and that $\ell$ is bounded by $B>0$, $\mathcal{F}$ is Lipschitz relative to $\mathcal{G}$, and that $(\ell, P, \mathcal{F})$ satisfies the $\varepsilon$-weak central condition. Then $(\ell, \hat{P}, \mathcal{F})$ satisfies the $\varepsilon + \mathcal{O}\big (e^B \text{Rate}_m(\mathcal{G}, P_{X,W} ) \big )$-weak central condition with probability at least $1-\delta$. 
\end{prop}


Next, we consider the norm induced loss. In this case it is also possible to obtain obtain the weak central condition for the artificially augmented problem  $(\ell, \hat{P}, \mathcal{F} )$. \\
\vspace{-3pt}

\begin{prop}[Continuous weak label] \label{prop: regression strong implies epsilon weak}
Suppose that $\ell^\text{weak}(w,w') = \norm{w-w'}$ and that $\ell$ is bounded by $B >0$, $\mathcal{F}$ is $L$-Lipschitz relative to $\mathcal{G}$, and that $(\ell, P, \mathcal{F})$ satisfies the $\varepsilon$-weak central condition. Then $(\ell, \hat{P}, \mathcal{F})$ satisfies the $\varepsilon + \mathcal{O}\big (\sqrt{Le^B \text{Rate}_m(\mathcal{G}, P_{X,W} )}\big )$-weak central condition with probability at least $1-\delta$. 
\end{prop}
For both propositions, a slight modification of the proofs easily eliminates the $e^B$ term when $\text{Rate}_m(\mathcal{G}, P_{X,W} )  \leq \mathcal{O}(e^{-B})$. Since we typically consider the regime where $ \text{Rate}_m(\mathcal{G}, P_{X,W} )$ is close to zero, Propositions \ref{prop: categorical strong implies epsilon weak} and \ref{prop: regression strong implies epsilon weak} essentially say that replacing $P$ by $\hat{P}$ only increases the weak central condition parameter slightly. \\
\vspace{-3pt}

The next, and final, step in our argument is to obtain a generalization bound for ERM under the $\varepsilon$-weak central condition. Once we have this bound, one can obtain good generalization bounds for the learning problem $(\ell, \hat{P}, \mathcal{F})$ since the previous two propositions guarantee that it satisfies the weak central condition from some small $\varepsilon$. Combining this observation with the results from the previous section finally allows us to obtain generalization bounds on Algorithm \ref{algo:double-estimation} when $\text{Rate}_n(\mathcal{F}, \hat{P})$ is ERM. \\

For this final step, we assume that our strong predictor class $\mathcal{F}$ is parameterized by a vector in $\mathbb{R}^d$, and identify each $f$ with this parameter vector. We also assume that the parameters live in an $L_2$ ball of radius $R$. By Lagrangian duality this is equivalent to our learning algorithm being ERM with $L_2$-regularization for some regularization parameter.

\begin{prop}\label{prop: generalization for weak central condition}
Suppose $(\ell, Q, \mathcal{F})$ satisfies the $\varepsilon$-weak central condition, $\ell$ is bounded by $B>0$, each $\mathcal{F}$ is $L'$-Lipschitz in its parameters in the $\ell_2$ norm, $\mathcal{F}$ is contained in the Euclidean ball of radius $R$, and $\mathcal{Y}$ is compact. Then when $\text{Alg}_n(\mathcal{F},Q)$ is ERM, the excess risk $\mathbb{E}_{Q}  [\ell_{\hat{f}}(U) - \ell_{f^*}(U)  ]$ is bounded by,
\[  \mathcal{O} \bigl(V \frac{d \log(RL'/\varepsilon)  + \log(1/\delta)}{n} + V \varepsilon \bigr),  \]
with probability at least $1-\delta$, where $V = B + \varepsilon$. 
\end{prop}

Any parameterized class of functions that is continuously differentiable in its parameters satisfies the $L'$-Lipschitz requirement since we assume the parameters live in a closed ball of radius $R$. The $\mathcal{Y}$ compactness assumption can be dropped in the case where $ y \mapsto \ell(y,\cdot)$ is Lipschitz. \\

Observe that the bound in Proposition~\ref{prop: generalization for weak central condition} depends linearly on $d$, the number of parameters of $\mathcal{F}$. Since we consider the regime where $n$ is small, the user might use only a small model (e.g., a shallow network) to parameterize $\mathcal{F}$, so $d$ may not be too large. On the other hand, the bound is independent of the complexity of $\mathcal{G}$. This is important since the user may want to use a powerful model class for $g$ to profit from the bountiful amounts of weak labels.  \\

Proposition~\ref{prop: generalization for weak central condition} gives a generalization bound for any learning problem $(\ell, Q, \mathcal{F})$ satisfying the weak central condition, and may therefore be of interest in the theory of fast rates more broadly. For our purposes, however, we shall apply it only to the particular learning problem $(\ell, \hat{P}, \mathcal{F})$. In this case, the $\varepsilon$ shall depend on $\text{Rate}_m(\mathcal{G}, P_{X,W} ) $, yielding strong generalization bounds when $\hat{g}$ has low excess risk.
Combining Proposition~\ref{prop: generalization for weak central condition} with both of the two previous propositions yields fast rates guarantees (Theorem \ref{thm: fast rates}) for the double estimation algorithm (Algorithm \ref{algo:double-estimation}) for ERM. The final bound depends on the rate of learning for the weak task, and on the quantity of weak data available $m$. \\

\begin{theorem}[Main result]\label{thm: fast rates}
Suppose the assumptions of Proposition \ref{prop: generalization for weak central condition} hold, $(\ell, P, \mathcal{F} )$ satisfies the central condition, and that $\text{Rate}_m(\mathcal{G}, P_{X,W} ) = \mathcal{O}(m^{-\alpha})$. Then,  when $\text{Alg}_n(\mathcal{F},\hat{P})$ is ERM we obtain excess risk $\mathbb{E} _P [\ell_{\hat{h}}(X,Y) - \ell_{h^*}(X,Y) ]$ that is bounded by,
\begin{equation*}
\mathcal{O} \bigl(  \frac{  d \alpha \beta \log RL'n + \log \frac{1}{\delta}}{n} + \frac{ L  }{n^{\alpha\beta}}   \bigr),
\end{equation*}
with probability at least $1-\delta$, if either of the following conditions hold,
\begin{enumerate}
  \setlength{\itemsep}{0pt}
\item  $m = \Omega(n^{\beta})$ and  $\ell^\text{weak}(w,w') = \mathbb{1}\{ w \neq w'\}$ (discrete $\mathcal{W}$-space).
\item $m = \Omega(n^{2\beta})$  and $\ell^\text{weak}(w,w') = \norm{w-w'}$ (continuous $\mathcal{W}$-space).
\end{enumerate}
\end{theorem}
To reduce clutter we absorb the dependence on $B$ into the big-$\mathcal{O}$. One can obtain similar bounds if the weak central condition holds but with an extra additive term in the bound.




\vspace{-5pt}
\section{Experiments} 

\begin{figure*}[t]
  \centering
  \label{figure:5}\includegraphics[width=70mm]{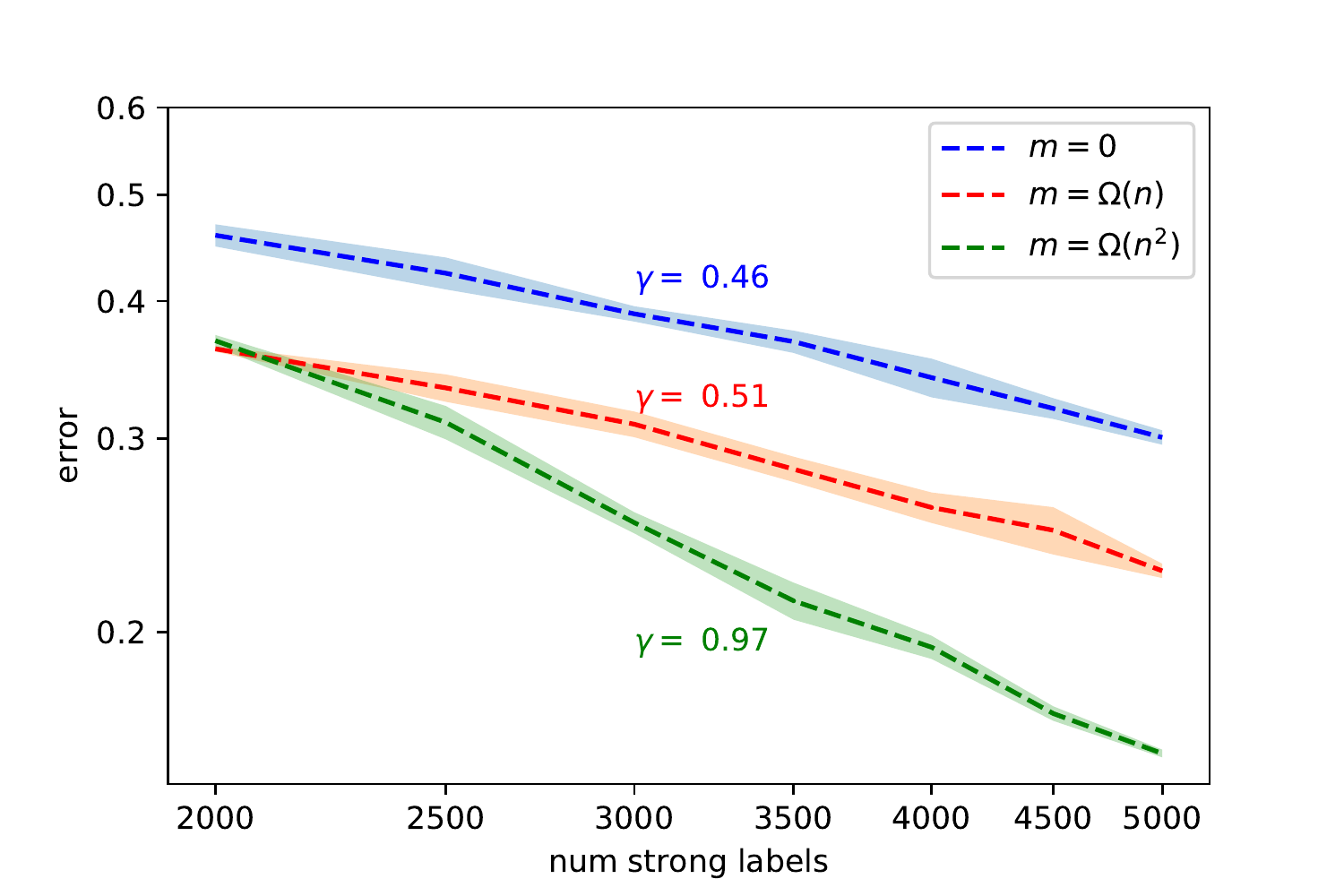}
  \label{figure:6}\includegraphics[width=70mm]{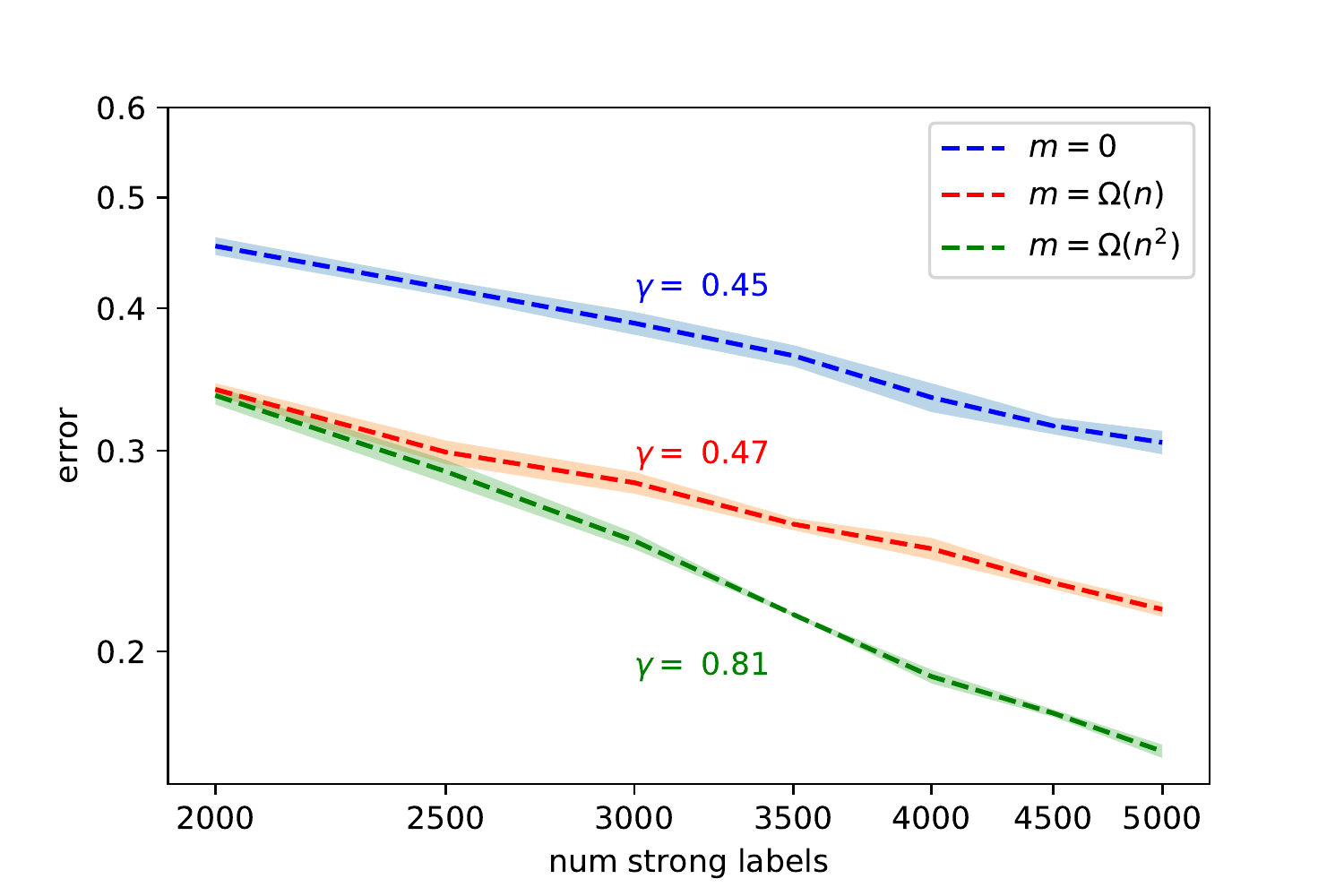}
   \begin{adjustwidth}{50pt}{50pt}
  \caption{Generalization error on CIFAR-$10$  using noisy weak labels for different growth rates of $m$. Left hand diagram is for simulated ``noisy labeler'', the right hand picture is for random noise.}%
  \label{fig: cifar10 noisy}
    \end{adjustwidth}
 \end{figure*}
 
   \begin{figure*}[h]
  \centering
  \includegraphics[width=40mm]{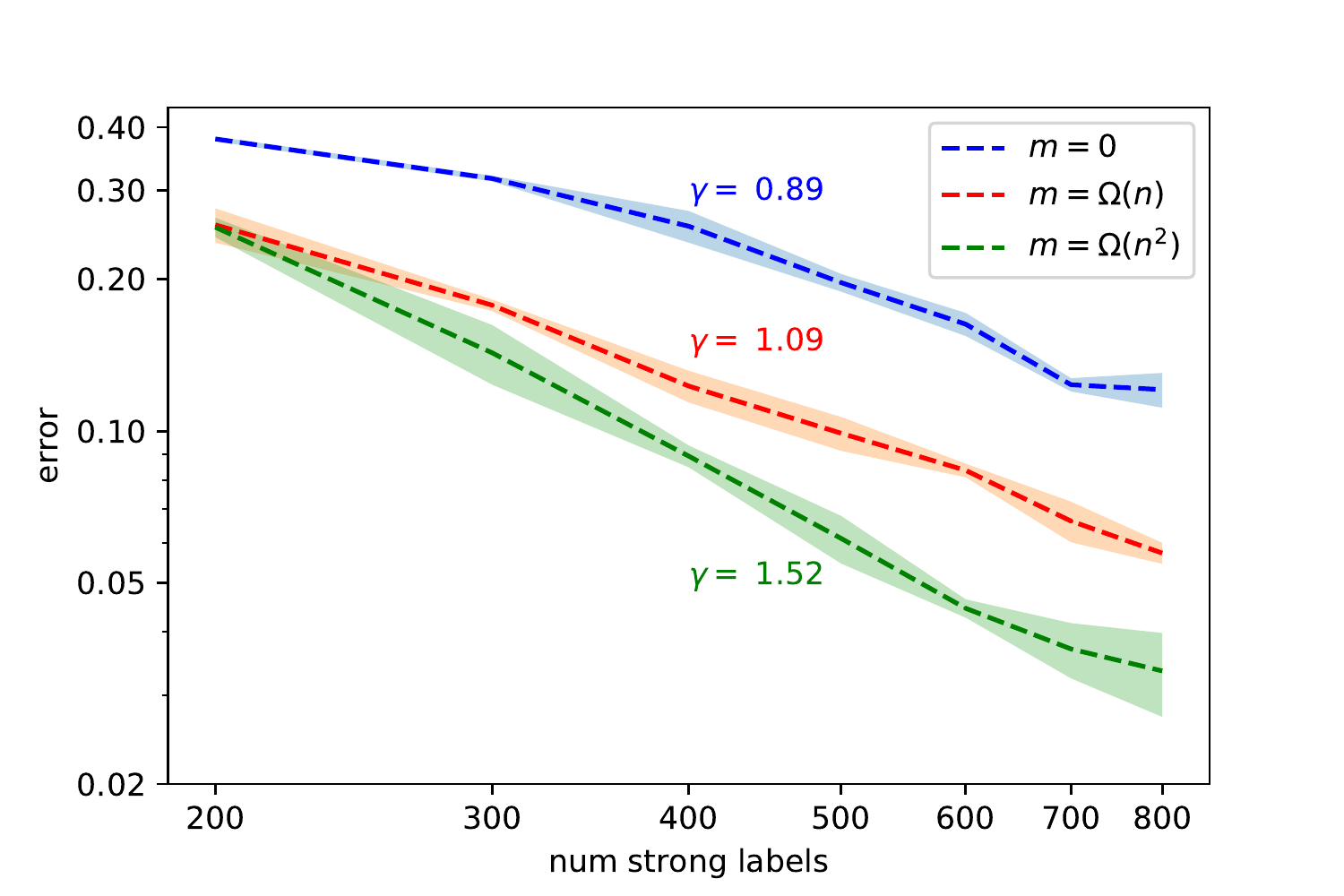}
  \label{figure:2}\includegraphics[width=40mm]{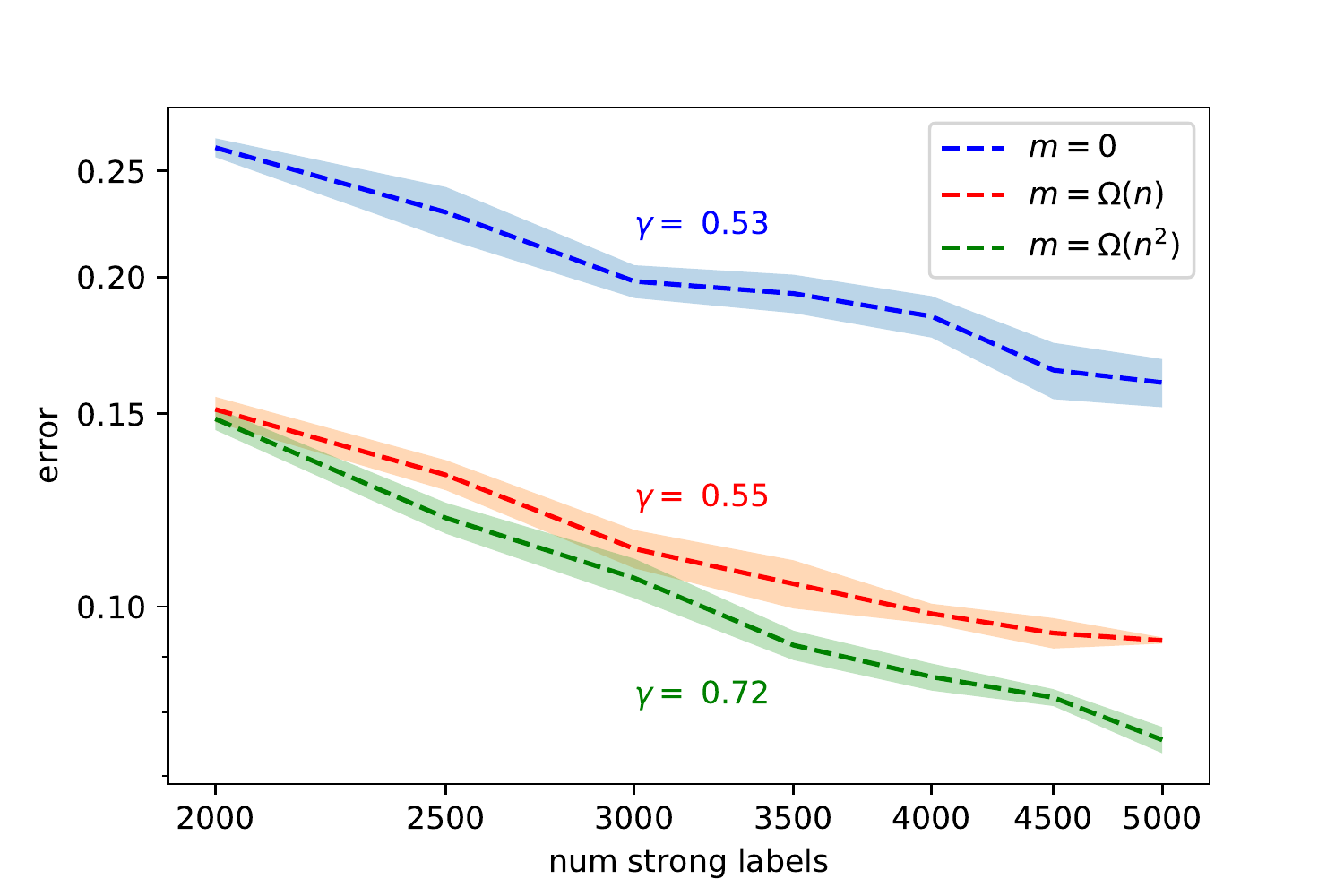}
  \label{figure:3}\includegraphics[width=40mm]{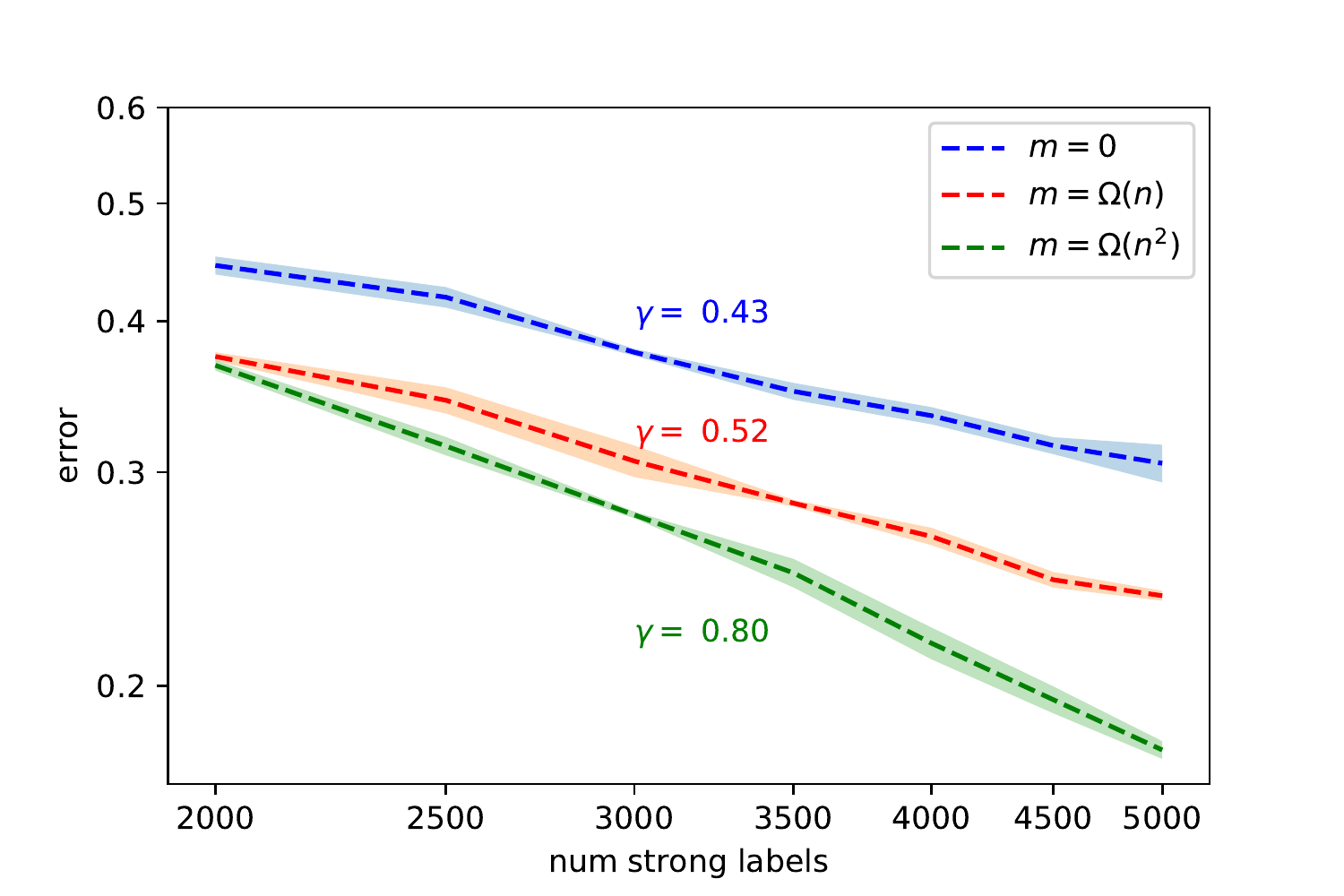}
  \begin{adjustwidth}{50pt}{50pt}
  \caption{Coarse labels.
  Generalization error on various datasets using coarse weak label grouping for different growth rates of $m$. Datasets left to right:  MNIST, SVHN, and CIFAR-$10$.}%
    \label{fig: mnist}
  \end{adjustwidth}
 \end{figure*} 
 \vspace{+3pt}

 Note that an excess risk bound of $b = C/n^\gamma$ implies a log-linear relationship $\log b = \log C - \gamma \log n$ between the error and amount of strong data. We are therefore able to visually interpret the learning rate $\gamma$ using a log-log scale as the \emph{negative of the gradient}. We experimentally study two types of weak label: noisy, and coarse labels. 
We study two cases: when the amount of weak data grows linearly with the amount of strong data, and when the amount of weak data grows quadratically with the amount of strong data (plus a baseline). All experiments use either a ResNet-$18$ or ResNet-$34$ for the weak feature map $g$. Full details of hyperparameter choices, architecture choices, and other experimental information are given in Appendix \ref{ap: experiments}. 

\vspace{-5pt}
\paragraph{Choice of baseline}

The aim of our experiments is to empirically study the relationship between generalization, weak dataset size, strong dataset size, and weak learning rate that our theoretical analysis predicts. Therefore, the clearest baseline comparison for Algorithm \ref{algo:double-estimation} is to vanilla supervised learning (i.e. $m=0$). 
\vspace{-5pt}

\subsection{Noisy Labels}

\paragraph{Simulated noisy labeler } 

First, we simulate a noisy labeler, who gets some examples wrong but in a way that is dependent on the example (as opposed to independent random noise). For example, think of a human annotator working on a crowd sourcing platform. We simulate noisy labelers by training an auxiliary deep network on a held out dataset to classify at a certain accuracy - for our CIFAR-$10$ experiments we train to $90\%$ accuracy. We then use the predictions of the auxiliary network as weak labels. The results are given in left hand part of Figure \ref{fig: cifar10 noisy}.

\vspace{-5pt}

\paragraph{Random noise} 
Second, we run experiments using independent random noise. To align with the simulated noisy labeler, we keep a given label the same with $90\%$ chance, and otherwise swap the label to any other label with equal chance (including back to itself). The results are given in right hand part of Figure \ref{fig: cifar10 noisy}. \\

In each case, both the generalization error when using additional weak data is lower, \emph{and} the learning rate itself is higher. Indeed, the learning rate improvement is significant. For simulated noisy labels, $\gamma = 0.46$ when $m=0$, and  $\gamma = 0.97$ for $m=\Omega(n^2)$. Random noisy labels has a similar result with $\gamma = 0.45$ and $\gamma = 0.81$ for $m=0$, and $m=\Omega(n^2)$ respectively.

\subsection{Coarse labels}

\paragraph{CIFAR-100 - concept clustering} To study learning with coarse weak labels, we first consider CIFAR-$100$. This dataset provides ready-made weak supervision. There are $100$ categories, which are clustered into $20$ super-categories each corresponding to a semantically meaningful collection. Each super category has exactly $5$ categories in each super-category. For example, the categories``maple'', ``oak'', ``palm'', ``pine'', and ``willow'' are all part of the super-category``trees''. We use the coarse super category as a weak label, and the fine grained $100$-way classes as strong labels. The results are presented in Figure \ref{fig: cifar100}.

\paragraph{Simple grouping} We also ran experiments using a  simple grouping to form weak labels for MNIST, SVHN, and CIFAR-$10$. We construct a weakly labeled dataset from MNIST and SVHN by assigning the weak label $W = Y ( \text{mod } d)$ for some $d \in \{2, \ldots, 10\}$. For CIFAR-$10$ we followed an analogous approach, forming a five different weak labels by grouping the ten strong labels into pairs. The results depicted in Figure \ref{fig: mnist} are all for $d=5$, however similar tests for different values of $d$, obtained similar results.  \\

The coarse label results are a similar story to noisy labels. Generalization error is consistently lower, and learning rate constantly high for larger $m$ growth rate. The differences are generally very significant, e.g. for CIFAR-$100$ where top-$1$ accuracy learning rate is $\gamma = 0.45$ for $m=0$, and  $\gamma = 0.70$ for $m=\Omega(n^2)$, and for MNIST  $\gamma = 0.89$ and $\gamma = 1.52$ for $m=0$ and  $m=\Omega(n^2)$ respectively.

 \begin{figure}[h]
  \centering
  \label{figure:5}\includegraphics[width=70mm]{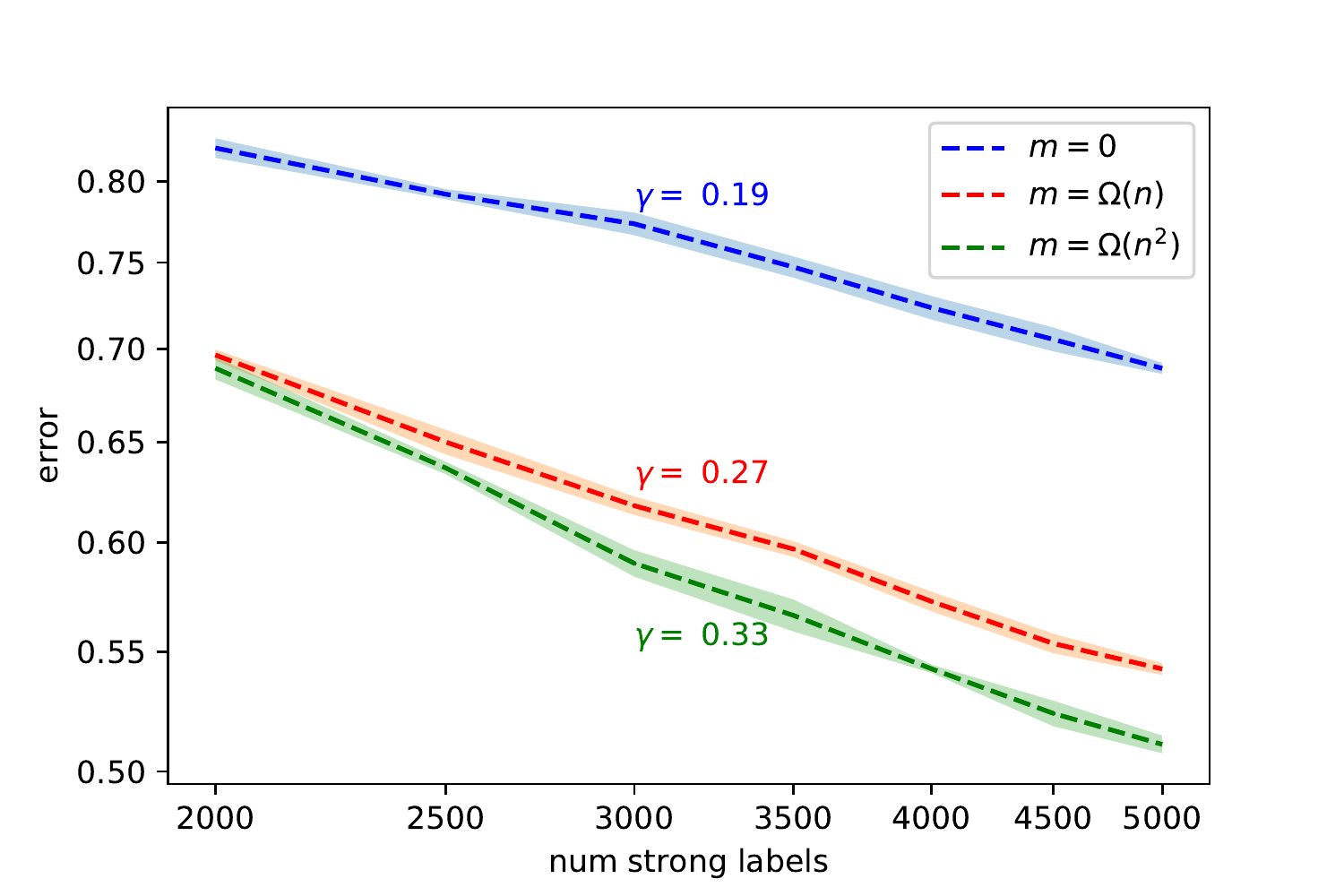}  
  \label{figure:6}\includegraphics[width=70mm]{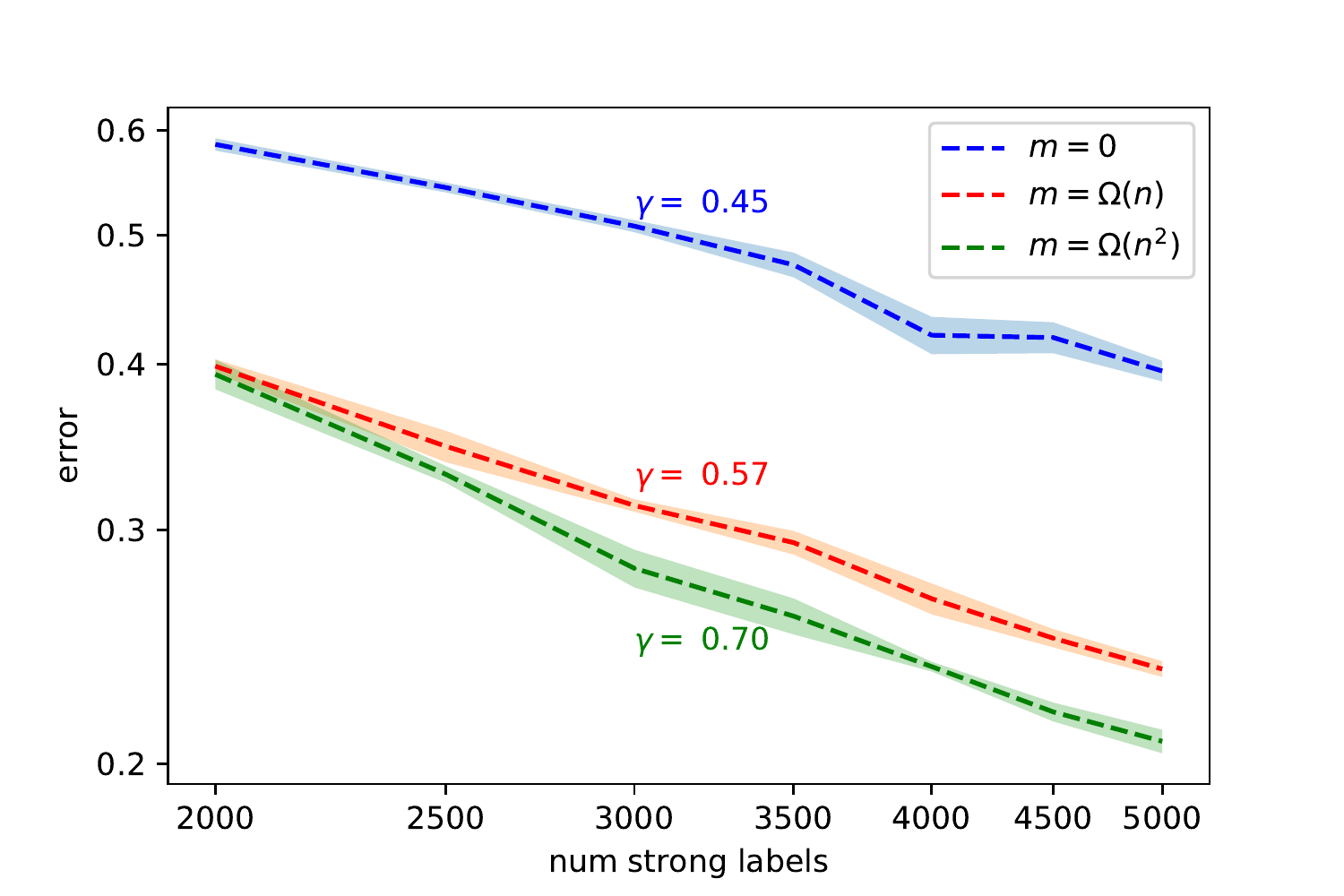}
    \begin{adjustwidth}{15pt}{15pt}
    \caption{Generalization error on CIFAR-$100$  using coarse weak labels for different growth rates of $m$. Top diagram is top-$1$ accuracy, and bottom diagram is top-$5$ accuracy. }%
        \label{fig: cifar100}
      \end{adjustwidth}
 \end{figure}

\section{Related Work}
\paragraph{Weakly supervised learning.}
There exists previous work on the case where one \emph{only} has weak labels. \citet{khetan2017learning} consider crowd sourced labels and use an EM-style algorithm to model the quality of individual workers. Another approach proposed in \cite{ratner2016data,ratner2019training} uses correlation between multiple different weak label sources to estimate the ground truth label. A different approach is to use pairwise semantic (dis)similarity as a form of weak signal about unlabeled data \cite{arora2019theoretical} or to use complementary labels, which give you a label telling you a class that the input is \emph{not} in \cite{xu2019generative}.

\vspace{-8pt}
\paragraph{Fast rates.}
There is a large body of work studying a variety of favorable situations under which it is possible to obtain rates better than slow-rates. From a generalization and optimization perspective, strongly convex losses enable fast rates for generalization and for fast convergence of stochastic gradient \cite{kakade2009generalization, hazan2007logarithmic, foster2019orthogonal}.  These works are special cases of exponentially concave learning, which is itself a special case of the central condition. There are completely different lines of work on fast rates, such as developing data dependent local Rademacher averages \cite{bartlett2005local}; and herding, which has been used to obtain fast rates for integral approximation \cite{welling2009herding}. 

\vspace{-8pt}
\paragraph{Learning with a nuisance component.} 
The two-step estimation algorithm we study in this paper is closely related to statistical learning under a nuisance component \cite{chernozhukov2018double, foster2019orthogonal}. In that setting one wishes to obtain excess risk bounds for the model $\hat{f}( \cdot , g_0(\cdot))$ where $W = g_0(X)$ is the true weak predictor. The analysis of learning in such settings rests crucially on the Neyman orthogonality assumption \cite{neyman}. Our setting has the important difference of seeking excess risk bounds for the compositional model $\hat{f}(\cdot, \hat{g}(\cdot))$.

\vspace{-8pt}
\paragraph{Self-supervised learning.} In self-supervised learning the user artificially constructs pretext learning problems based on attributes of unlabeled data  \cite{doersch2015unsupervised, gidaris2018unsupervised}. 
In other words, it is often possible to construct a weakly supervised learning problem where the choice of weak labels are a design choice of the user. In line with our analysis, the success of self-supervised representations relies on picking pretext labels that capture useful information about the strong label such as invariances and spacial understanding \cite{noroozi2016unsupervised, misra2019self}. Conversely, weakly supervised learning can be viewed as a special case of self-supervision where the pretext task is selected from some naturally occurring label source \cite{jing2019self}.


\vspace{-3pt}
\section{Discussion}
\label{sec:discussion}

Our work focuses on analyzing weakly supervised learning. We believe, however, that the same framework could be used to analyze other popular learning paradigms. One immediate possibility is to extend our analysis to consider multiple inconsistent sources of weak labels as in~\citep{khetan2017learning}. Other important extensions would be to include self-supervised learning and pre-training. The key technical difference between these settings and ours is that in both these settings the marginal distribution of features $P(X)$ is potentially different on the pretext task as compared to the downstream tasks of interest. Cases where the marginal $P(X)$ does not shift fall within the scope of our analysis. \\

Another option is to use our representation transfer analysis to study multi-task or meta-learning settings where one wishes to reuse an embedding across multiple tasks with shared characteristics with the aim of obtaining certified performance across all tasks. Finally, a completely different direction, based on the observation that our  analysis is predicated on the idea of ``cheap'' weak labels and ``costly'' strong labels, is to ask how best to allocate a finite budget for label collection when faced with varying quality label sources.

\small
\pagebreak
\setlength{\bibsep}{3pt}
\bibliographystyle{icml2020}
\bibliography{bibliography}

\begin{thebibliography}{67}
\providecommand{\natexlab}[1]{#1}
\providecommand{\url}[1]{\texttt{#1}}
\expandafter\ifx\csname urlstyle\endcsname\relax
  \providecommand{\doi}[1]{doi: #1}\else
  \providecommand{\doi}{doi: \begingroup \urlstyle{rm}\Url}\fi

\bibitem[Arora et~al.(2019)Arora, Khandeparkar, Khodak, Plevrakis, and
  Saunshi]{arora2019theoretical}
Arora, S., Khandeparkar, H., Khodak, M., Plevrakis, O., and Saunshi, N.
\newblock A theoretical analysis of contrastive unsupervised representation
  learning.
\newblock In \emph{Int. Conference on Machine Learning (ICML)}, 2019.

\bibitem[Audibert et~al.(2009)]{audibert2009fast}
Audibert, J.-Y. et~al.
\newblock Fast learning rates in statistical inference through aggregation.
\newblock \emph{The Annals of Statistics}, 37\penalty0 (4), 2009.

\bibitem[Bartlett \& Mendelson(2006)Bartlett and
  Mendelson]{bartlett2006empirical}
Bartlett, P.~L. and Mendelson, S.
\newblock Empirical minimization.
\newblock \emph{Probability theory and related fields}, 135\penalty0 (3), 2006.

\bibitem[Bartlett et~al.(2005)Bartlett, Bousquet, Mendelson,
  et~al.]{bartlett2005local}
Bartlett, P.~L., Bousquet, O., Mendelson, S., et~al.
\newblock Local {R}ademacher complexities.
\newblock \emph{The Annals of Statistics}, 33\penalty0 (4), 2005.

\bibitem[Bilen \& Vedaldi(2016)Bilen and Vedaldi]{bilen2016weakly}
Bilen, H. and Vedaldi, A.
\newblock Weakly supervised deep detection networks.
\newblock In \emph{IEEE Conference on Computer Vision and Pattern Recognition
  (CVPR)}, 2016.

\bibitem[Branson et~al.(2017)Branson, Van~Horn, and Perona]{branson2017lean}
Branson, S., Van~Horn, G., and Perona, P.
\newblock Lean crowdsourcing: {C}ombining humans and machines in an online
  system.
\newblock In \emph{IEEE Conference on Computer Vision and Pattern Recognition
  (CVPR)}, 2017.

\bibitem[Carl \& Stephani(1990)Carl and Stephani]{carl1990entropy}
Carl, B. and Stephani, I.
\newblock \emph{Entropy, compactness and the approximation of operators}.
\newblock Number~98. Cambridge University Press, 1990.

\bibitem[Carreira \& Zisserman(2017)Carreira and Zisserman]{carreira17}
Carreira, J. and Zisserman, A.
\newblock Quo vadis, action recognition? {A} new model and the kinetics
  dataset.
\newblock In \emph{IEEE Conference on Computer Vision and Pattern Recognition
  (CVPR)}, 2017.

\bibitem[Chen \& Batmanghelich(2019)Chen and Batmanghelich]{chen2019weakly}
Chen, J. and Batmanghelich, K.
\newblock Weakly supervised disentanglement by pairwise similarities.
\newblock In \emph{Association for the Advancement of Artificial Intelligence
  (AAAI)}, 2019.

\bibitem[Chernozhukov et~al.(2018)Chernozhukov, Chetverikov, Demirer, Duflo,
  Hansen, Newey, and Robins]{chernozhukov2018double}
Chernozhukov, V., Chetverikov, D., Demirer, M., Duflo, E., Hansen, C., Newey,
  W., and Robins, J.
\newblock Double/debiased machine learning for treatment and structural
  parameters, 2018.

\bibitem[Dai et~al.(2007)Dai, Yang, Xue, and Yu]{dai2007boosting}
Dai, W., Yang, Q., Xue, G.-R., and Yu, Y.
\newblock Boosting for transfer learning.
\newblock In \emph{Int. Conference on Machine Learning (ICML)}, 2007.

\bibitem[Dalalyan et~al.(2012)Dalalyan, Tsybakov, et~al.]{dalalyan2012mirror}
Dalalyan, A.~S., Tsybakov, A.~B., et~al.
\newblock Mirror averaging with sparsity priors.
\newblock \emph{Bernoulli}, 18\penalty0 (3), 2012.

\bibitem[Desrosiers \& Karypis(2011)Desrosiers and
  Karypis]{desrosiers2011comprehensive}
Desrosiers, C. and Karypis, G.
\newblock A comprehensive survey of neighborhood-based recommendation methods.
\newblock In \emph{Recommender systems handbook}. Springer, 2011.

\bibitem[Doersch et~al.(2015)Doersch, Gupta, and
  Efros]{doersch2015unsupervised}
Doersch, C., Gupta, A., and Efros, A.~A.
\newblock Unsupervised visual representation learning by context prediction.
\newblock In \emph{Int. Conference on Computer Vision (ICCV)}, 2015.

\bibitem[Donahue et~al.(2014)Donahue, Jia, Vinyals, Hoffman, Zhang, Tzeng, and
  Darrell]{donahue14}
Donahue, J., Jia, Y., Vinyals, O., Hoffman, J., Zhang, N., Tzeng, E., and
  Darrell, T.
\newblock Decaf: A deep convolutional activation feature for generic visual
  recognition.
\newblock In \emph{Int. Conference on Machine Learning (ICML)}, 2014.

\bibitem[Durand et~al.(2017)Durand, Mordan, Thome, and Cord]{durand2017wildcat}
Durand, T., Mordan, T., Thome, N., and Cord, M.
\newblock Wildcat: {W}eakly supervised learning of deep convnets for image
  classification, pointwise localization and segmentation.
\newblock In \emph{IEEE Conference on Computer Vision and Pattern Recognition
  (CVPR)}, 2017.

\bibitem[Foster \& Syrgkanis(2019)Foster and Syrgkanis]{foster2019orthogonal}
Foster, D.~J. and Syrgkanis, V.
\newblock Orthogonal statistical learning.
\newblock In \emph{Conference on Learning Theory (COLT)}, 2019.

\bibitem[Foster et~al.(2018)Foster, Kale, Luo, Mohri, and
  Sridharan]{foster2018logistic}
Foster, D.~J., Kale, S., Luo, H., Mohri, M., and Sridharan, K.
\newblock Logistic regression: The importance of being improper.
\newblock In \emph{Conference on Learning Theory (COLT)}, 2018.

\bibitem[Fries et~al.(2019)Fries, Varma, Chen, Xiao, Tejeda, Saha, Dunnmon,
  Chubb, Maskatia, Fiterau, et~al.]{fries2019weakly}
Fries, J.~A., Varma, P., Chen, V.~S., Xiao, K., Tejeda, H., Saha, P., Dunnmon,
  J., Chubb, H., Maskatia, S., Fiterau, M., et~al.
\newblock Weakly supervised classification of aortic valve malformations using
  unlabeled cardiac mri sequences.
\newblock \emph{Nature communications}, 10\penalty0 (1), 2019.

\bibitem[Gidaris et~al.(2018)Gidaris, Singh, and
  Komodakis]{gidaris2018unsupervised}
Gidaris, S., Singh, P., and Komodakis, N.
\newblock Unsupervised representation learning by predicting image rotations.
\newblock \emph{preprint arXiv:1803.07728}, 2018.

\bibitem[Girshick et~al.(2014)Girshick, Donahue, Darrell, and
  Malik]{girshick14}
Girshick, R., Donahue, J., Darrell, T., and Malik, J.
\newblock Rich feature hierarchies for accurate object detection and semantic
  segmentation.
\newblock In \emph{IEEE Conference on Computer Vision and Pattern Recognition
  (CVPR)}, 2014.

\bibitem[Guo et~al.(2018)Guo, Liu, Bakker, Guo, and Lew]{guo2018cnn}
Guo, Y., Liu, Y., Bakker, E.~M., Guo, Y., and Lew, M.~S.
\newblock {CNN-RNN}: {A} large-scale hierarchical image classification
  framework.
\newblock \emph{Multimedia Tools and Applications}, 77\penalty0 (8), 2018.

\bibitem[Hazan et~al.(2007)Hazan, Agarwal, and Kale]{hazan2007logarithmic}
Hazan, E., Agarwal, A., and Kale, S.
\newblock Logarithmic regret algorithms for online convex optimization.
\newblock \emph{Machine Learning}, 69\penalty0 (2-3), 2007.

\bibitem[He et~al.(2016)He, Zhang, Ren, and Sun]{he2016deep}
He, K., Zhang, X., Ren, S., and Sun, J.
\newblock Deep residual learning for image recognition.
\newblock In \emph{IEEE Conference on Computer Vision and Pattern Recognition
  (CVPR)}, 2016.

\bibitem[Huh et~al.(2016)Huh, Agrawal, and Efros]{huh2016makes}
Huh, M., Agrawal, P., and Efros, A.~A.
\newblock What makes {I}mage{N}et good for transfer learning?
\newblock \emph{preprint arXiv:1608.08614}, 2016.

\bibitem[Jing \& Tian(2019)Jing and Tian]{jing2019self}
Jing, L. and Tian, Y.
\newblock Self-supervised visual feature learning with deep neural networks: A
  survey.
\newblock \emph{preprint arXiv:1902.06162}, 2019.

\bibitem[Juditsky et~al.(2008)Juditsky, Rigollet, Tsybakov,
  et~al.]{juditsky2008learning}
Juditsky, A., Rigollet, P., Tsybakov, A.~B., et~al.
\newblock Learning by mirror averaging.
\newblock \emph{The Annals of Statistics}, 36\penalty0 (5), 2008.

\bibitem[Kakade \& Tewari(2009)Kakade and Tewari]{kakade2009generalization}
Kakade, S.~M. and Tewari, A.
\newblock On the generalization ability of online strongly convex programming
  algorithms.
\newblock In \emph{Advances in Neural Information Processing Systems
  (NeurIPS)}, 2009.

\bibitem[Khetan et~al.(2018)Khetan, Lipton, and Anandkumar]{khetan2017learning}
Khetan, A., Lipton, Z.~C., and Anandkumar, A.
\newblock Learning from noisy singly-labeled data.
\newblock \emph{Int. Conf. on Learning Representations (ICLR)}, 2018.

\bibitem[Kingma \& Ba(2015)Kingma and Ba]{kingma2014adam}
Kingma, D.~P. and Ba, J.
\newblock Adam: A method for stochastic optimization.
\newblock In \emph{Int. Conf. on Learning Representations (ICLR)}, 2015.

\bibitem[Kleindessner \& Awasthi(2018)Kleindessner and
  Awasthi]{kleindessner2018crowdsourcing}
Kleindessner, M. and Awasthi, P.
\newblock Crowdsourcing with arbitrary adversaries.
\newblock In \emph{Int. Conference on Machine Learning (ICML)}, 2018.

\bibitem[Lecu{\'e} et~al.(2014)Lecu{\'e}, Rigollet, et~al.]{lecue2014optimal}
Lecu{\'e}, G., Rigollet, P., et~al.
\newblock Optimal learning with {Q}-aggregation.
\newblock \emph{The Annals of Statistics}, 42\penalty0 (1), 2014.

\bibitem[Li et~al.(2017)Li, Wang, Li, Agustsson, and van Gool]{li2017webvision}
Li, W., Wang, L., Li, W., Agustsson, E., and van Gool, L.
\newblock Webvision database: Visual learning and understanding from web data.
\newblock In \emph{preprint arXiv:1708.02862}, 2017.

\bibitem[Liao et~al.(2005)Liao, Xue, and Carin]{liao2005logistic}
Liao, X., Xue, Y., and Carin, L.
\newblock Logistic regression with an auxiliary data source.
\newblock In \emph{Int. Conference on Machine Learning (ICML)}, 2005.

\bibitem[Mahajan et~al.(2018)Mahajan, Girshick, Ramanathan, He, Paluri, Li,
  Bharambe, and van~der Maaten]{mahajan2018exploring}
Mahajan, D., Girshick, R., Ramanathan, V., He, K., Paluri, M., Li, Y.,
  Bharambe, A., and van~der Maaten, L.
\newblock Exploring the limits of weakly supervised pretraining.
\newblock In \emph{Europ. Conference on Computer Vision (ECCV)}, 2018.

\bibitem[Medlock \& Briscoe(2007)Medlock and Briscoe]{medlock2007weakly}
Medlock, B. and Briscoe, T.
\newblock Weakly supervised learning for hedge classification in scientific
  literature.
\newblock In \emph{Proceedings of the 45th annual meeting of the association of
  computational linguistics}, 2007.

\bibitem[Mehta(2016)]{mehta2016fast}
Mehta, N.~A.
\newblock Fast rates with high probability in exp-concave statistical learning.
\newblock In \emph{Proc. Int. Conference on Artificial Intelligence and
  Statistics (AISTATS)}, 2016.

\bibitem[Misra \& van~der Maaten(2019)Misra and van~der Maaten]{misra2019self}
Misra, I. and van~der Maaten, L.
\newblock Self-supervised learning of pretext-invariant representations.
\newblock \emph{preprint arXiv:1912.01991}, 2019.

\bibitem[Naumov et~al.(2019)Naumov, Mudigere, Shi, Huang, Sundaraman, Park,
  Wang, Gupta, Wu, Azzolini, et~al.]{naumov2019deep}
Naumov, M., Mudigere, D., Shi, H.-J.~M., Huang, J., Sundaraman, N., Park, J.,
  Wang, X., Gupta, U., Wu, C.-J., Azzolini, A.~G., et~al.
\newblock Deep learning recommendation model for personalization and
  recommendation systems.
\newblock \emph{preprint arXiv:1906.00091}, 2019.

\bibitem[Neyman \& Scott(1965)Neyman and Scott]{neyman}
Neyman, J. and Scott, E.~L.
\newblock Asymptotically optimal tests of composite hypotheses for randomized
  experiments with noncontrolled predictor variables.
\newblock \emph{Journal of the American Statistical Association}, 60\penalty0
  (311), 1965.

\bibitem[Noroozi \& Favaro(2016)Noroozi and Favaro]{noroozi2016unsupervised}
Noroozi, M. and Favaro, P.
\newblock Unsupervised learning of visual representations by solving jigsaw
  puzzles.
\newblock In \emph{Europ. Conference on Computer Vision (ECCV)}. Springer,
  2016.

\bibitem[Oquab et~al.(2015)Oquab, Bottou, Laptev, and Sivic]{oquab2015object}
Oquab, M., Bottou, L., Laptev, I., and Sivic, J.
\newblock Is object localization for free?-weakly-supervised learning with
  convolutional neural networks.
\newblock In \emph{IEEE Conference on Computer Vision and Pattern Recognition
  (CVPR)}, 2015.

\bibitem[Pan \& Yang(2009)Pan and Yang]{pan2009survey}
Pan, S.~J. and Yang, Q.
\newblock A survey on transfer learning.
\newblock \emph{IEEE Transactions on knowledge and data engineering},
  22\penalty0 (10), 2009.

\bibitem[Paszke et~al.(2019)Paszke, Gross, Massa, Lerer, Bradbury, Chanan,
  Killeen, Lin, Gimelshein, Antiga, et~al.]{paszke2019pytorch}
Paszke, A., Gross, S., Massa, F., Lerer, A., Bradbury, J., Chanan, G., Killeen,
  T., Lin, Z., Gimelshein, N., Antiga, L., et~al.
\newblock Py{T}orch: An imperative style, high-performance deep learning
  library.
\newblock In \emph{Advances in Neural Information Processing Systems
  (NeurIPS)}, 2019.

\bibitem[Ratner et~al.(2019)Ratner, Hancock, Dunnmon, Sala, Pandey, and
  R{\'e}]{ratner2019training}
Ratner, A., Hancock, B., Dunnmon, J., Sala, F., Pandey, S., and R{\'e}, C.
\newblock Training complex models with multi-task weak supervision.
\newblock In \emph{Association for the Advancement of Artificial Intelligence
  (AAAI)}, volume~33, 2019.

\bibitem[Ratner et~al.(2016)Ratner, De~Sa, Wu, Selsam, and
  R{\'e}]{ratner2016data}
Ratner, A.~J., De~Sa, C.~M., Wu, S., Selsam, D., and R{\'e}, C.
\newblock Data programming: Creating large training sets, quickly.
\newblock In \emph{Advances in Neural Information Processing Systems
  (NeurIPS)}, 2016.

\bibitem[Ricci et~al.(2011)Ricci, Rokach, and Shapira]{ricci2011introduction}
Ricci, F., Rokach, L., and Shapira, B.
\newblock Introduction to recommender systems handbook.
\newblock In \emph{Recommender systems handbook}. Springer, 2011.

\bibitem[Schork(2015)]{schork2015personalized}
Schork, N.~J.
\newblock Personalized medicine: time for one-person trials.
\newblock \emph{Nature}, 520\penalty0 (7549), 2015.

\bibitem[Sun et~al.(2017)Sun, Shrivastava, Singh, and Gupta]{sun17}
Sun, C., Shrivastava, A., Singh, S., and Gupta, A.
\newblock Revisiting unreasonable effectiveness of data in deep learning era.
\newblock In \emph{Int. Conference on Computer Vision (ICCV)}, 2017.

\bibitem[Taherkhani et~al.(2019)Taherkhani, Kazemi, Dabouei, Dawson, and
  Nasrabadi]{taherkhani2019weakly}
Taherkhani, F., Kazemi, H., Dabouei, A., Dawson, J., and Nasrabadi, N.~M.
\newblock A weakly supervised fine label classifier enhanced by coarse
  supervision.
\newblock In \emph{Int. Conference on Computer Vision (ICCV)}, 2019.

\bibitem[Tsybakov et~al.(2004)]{tsybakov2004optimal}
Tsybakov, A.~B. et~al.
\newblock Optimal aggregation of classifiers in statistical learning.
\newblock \emph{The Annals of Statistics}, 32\penalty0 (1), 2004.

\bibitem[van Erven et~al.(2012)van Erven, Gr{\"u}nwald, Reid, and
  Williamson]{erven2012mixability}
van Erven, T., Gr{\"u}nwald, P., Reid, M., and Williamson, R.
\newblock Mixability in statistical learning.
\newblock In \emph{Advances in Neural Information Processing Systems
  (NeurIPS)}, 2012.

\bibitem[van Erven et~al.(2015)van Erven, Grunwald, Mehta, Reid, Williamson,
  et~al.]{van2015fast}
van Erven, T., Grunwald, P., Mehta, N.~A., Reid, M., Williamson, R., et~al.
\newblock Fast rates in statistical and online learning.
\newblock In \emph{Journal of Machine Learning Research}, 2015.

\bibitem[Vapnik \& Chervonenkis(1974)Vapnik and Chervonenkis]{vapnik1974theory}
Vapnik, V. and Chervonenkis, A.
\newblock Theory of pattern recognition, 1974.

\bibitem[Vapnik \& Chervonenkis(1971)Vapnik and Chervonenkis]{vapnik:264}
Vapnik, V.~N. and Chervonenkis, A.~Y.
\newblock On the uniform convergence of relative frequencies of events to their
  probabilities.
\newblock \emph{Theory of Probability and its Applications}, 16\penalty0 (2),
  1971.

\bibitem[Vovk(1998)]{vovk1998game}
Vovk, V.
\newblock A game of prediction with expert advice.
\newblock \emph{Journal of Computer and System Sciences}, 56\penalty0 (2),
  1998.

\bibitem[Vovk(1990)]{vovk1990aggregating}
Vovk, V.~G.
\newblock Aggregating strategies.
\newblock \emph{Proc. of Computational Learning Theory, 1990}, 1990.

\bibitem[Wan et~al.(2018)Wan, Wei, Jiao, Han, and Ye]{wan2018min}
Wan, F., Wei, P., Jiao, J., Han, Z., and Ye, Q.
\newblock Min-entropy latent model for weakly supervised object detection.
\newblock In \emph{IEEE Conference on Computer Vision and Pattern Recognition
  (CVPR)}, 2018.

\bibitem[Wang et~al.(2019)Wang, Sohn, Liu, Shen, Wang, Atkinson, Amin, and
  Liu]{wang2019clinical}
Wang, Y., Sohn, S., Liu, S., Shen, F., Wang, L., Atkinson, E.~J., Amin, S., and
  Liu, H.
\newblock A clinical text classification paradigm using weak supervision and
  deep representation.
\newblock \emph{BMC medical informatics and decision making}, 19\penalty0 (1),
  2019.

\bibitem[Welling(2009)]{welling2009herding}
Welling, M.
\newblock Herding dynamical weights to learn.
\newblock In \emph{Int. Conference on Machine Learning (ICML)}, 2009.

\bibitem[Xu et~al.(2019)Xu, Gong, Chen, Liu, Zhang, and
  Batmanghelich]{xu2019generative}
Xu, Y., Gong, M., Chen, J., Liu, T., Zhang, K., and Batmanghelich, K.
\newblock Generative-discriminative complementary learning.
\newblock In \emph{preprint arXiv:1904.01612}, 2019.

\bibitem[Yan et~al.(2015)Yan, Zhang, Piramuthu, Jagadeesh, DeCoste, Di, and
  Yu]{yan2015hd}
Yan, Z., Zhang, H., Piramuthu, R., Jagadeesh, V., DeCoste, D., Di, W., and Yu,
  Y.
\newblock {HD-CNN}: hierarchical deep convolutional neural networks for large
  scale visual recognition.
\newblock In \emph{Int. Conference on Computer Vision (ICCV)}, 2015.

\bibitem[Zeiler \& Fergus(2014)Zeiler and Fergus]{zeiler14}
Zeiler, M.~D. and Fergus, R.
\newblock Visualizing and understanding convolutional neural networks.
\newblock In \emph{Europ. Conference on Computer Vision (ECCV)}, 2014.

\bibitem[Zhang et~al.(2014)Zhang, Chen, Zhou, and Jordan]{zhang2014spectral}
Zhang, Y., Chen, X., Zhou, D., and Jordan, M.~I.
\newblock Spectral methods meet {EM}: {A} provably optimal algorithm for
  crowdsourcing.
\newblock In \emph{Advances in Neural Information Processing Systems
  (NeurIPS)}, 2014.

\bibitem[Zhang \& Sabuncu(2018)Zhang and Sabuncu]{zhang2018generalized}
Zhang, Z. and Sabuncu, M.
\newblock Generalized cross entropy loss for training deep neural networks with
  noisy labels.
\newblock In \emph{Advances in Neural Information Processing Systems
  (NeurIPS)}, 2018.

\bibitem[Zhao et~al.(2011)Zhao, Li, and Xing]{zhao2011large}
Zhao, B., Li, F., and Xing, E.~P.
\newblock Large-scale category structure aware image categorization.
\newblock In \emph{Advances in Neural Information Processing Systems
  (NeurIPS)}, 2011.

\bibitem[Zhou(2018)]{zhou2018brief}
Zhou, Z.-H.
\newblock A brief introduction to weakly supervised learning.
\newblock \emph{National Science Review}, 5\penalty0 (1), 2018.

\end{thebibliography}

\clearpage

\onecolumn
\appendix

\section{Section \ref{sec: analysis} Proofs}

We begin by obtaining the decomposition that is instrumental in dividing the excess risk into two pieces that can be then studied separately. \\ 

\begin{prop}[Proposition \ref{thm: general bound}]

Suppose that $f^*$ is $L$-Lipschitz relative to $\mathcal{G}$. Then the excess risk  $\mathbb{E}  [\ell_{\hat{h}}(X,Y) - \ell_{h^*}(X,Y) ]$ bounded by,
\[ 2L \text{Rate}_m(\mathcal{G}, P_{X,W} ) +  \text{Rate}_n(\mathcal{F}, \hat{P}).  \]
\end{prop}
\begin{proof}[Proof of Proposition \ref{thm: general bound}]
Let us split the excess risk into three parts 
\begin{align*}
\mathbb{E} \big [\ell_{\hat{h}}(X,Y) - \ell_{h^*}(X,Y) \big  ]  = &\mathbb{E} \big [\ell_{ \hat{f}( \cdot , \hat{g} )}(X,Y) - \ell_{ f^*( \cdot , \hat{g} )}(X,Y) \big  ] \\
 &\qquad \qquad  +  \mathbb{E} \big [\ell_{ f^*( \cdot , \hat{g} )}(X,Y) - \ell_{ f^*( \cdot , g_0 )}(X,Y) \big  ] + \mathbb{E} \big [\ell_{ f^*( \cdot , g_0 )}(X,Y) - \ell_{ f^*( \cdot , g^{*} )}(X,Y) \big  ]. 
\end{align*}
By definition, the first term is bounded by $\text{Rate}_n(\mathcal{F}, \hat{P}) $. The relative Lipschitzness of $f^*$ delivers the following bound on the second and third terms respectively, 
\begin{align*}
\mathbb{E} \big [\ell_{ f^*( \cdot , \hat{g} )}(X,Y) - \ell_{ f^*( \cdot , g_0 )}(X,Y) \big ]
&\leq  L  \mathbb{E}_P \ell^{\text{weak}} \big ( \beta_{\hat{g}}^\top \hat{g}(X) , \beta_{g_0}^\top g_0(X)  \big ), \\
\mathbb{E} \big [\ell_{ f^*( \cdot , g_0 )}(X,Y) - \ell_{ f^*( \cdot , g^{*})}(X,Y) \big ]
&\leq  L  \mathbb{E}_P \ell^{\text{weak}} \big ( \beta_{g_0}^\top g_0(X), \beta_{g^{*}}^\top g^{*}(X)   \big ).
\end{align*}
Since $g^*$ attains minimal risk, and $W = \beta_{g_0}^\top g_0(X) $, the sum of these two terms  can be bounded by, 
\begin{align*}
2L   \mathbb{E}_P \ell^{\text{weak}} \big ( \beta_{\hat{g}}^\top \hat{g}(X) , W \big ) \leq 2L\text{Rate}_m(\mathcal{G}, P_{X,W} ). 
\end{align*}
Combining this with the bound on the first term yields the claim.

\end{proof}

The next two propositions show, for the two cases of $\ell^\text{weak}$ of interest, that the weak central condition is preserved (with a slight weakening in the constant) when replacing the population distribution $P$ by the distribution $\hat{P}$ obtained by replacing the true weak label $W$ by the learned weak estimate $\hat{g}(X)$. \\

\begin{prop}[Proposition \ref{prop: categorical strong implies epsilon weak}]
Suppose that $\ell^\text{weak}(w,w') = \mathbb{1}\{ w \neq w'\}$ and that $\ell$ is bounded by $B>0$, $\mathcal{F}$ is Lipschitz relative to $\mathcal{G}$, and that $(\ell, P, \mathcal{F})$ satisfies the $\varepsilon$-weak central condition. Then $(\ell, \hat{P}, \mathcal{F})$ satisfies the $\varepsilon + \mathcal{O}\big (\text{Rate}_m(\mathcal{G}, P_{X,W} ) \big )$-weak central condition with probability at least $1-\delta$. 
\end{prop}
\begin{proof}[Proof of Proposition \ref{prop: categorical strong implies epsilon weak}]
Note first that 
\[\frac{1}{\eta} \log \mathbb{E}_{\hat{P}} \exp {\big (- \eta (\ell_f - \ell_{f^*}) \big )}  =  \frac{1}{\eta} \log \mathbb{E}_{P} \exp{ \big (- \eta (\ell_{f(\cdot, \hat{g}) } - \ell_{f^*(\cdot, \hat{g})}) \big )} \]
 where we recall that we have overloaded the loss $\ell$ to include both $\ell_f$ and $\ell_h$.  To prove $(\ell, \hat{P}, \mathcal{F})$ satisfies the central condition we therefore need to bound $\frac{1}{\eta} \log \mathbb{E}_{P} \exp {\big (- \eta (\ell_{f(\cdot, \hat{g})} - \ell_{f^*(\cdot, \hat{g})}) \big )}$ above by some constant. We begin bounding (line by line explanations are below),
\begin{align*}
\begin{split}
\frac{1}{\eta} \log \mathbb{E}_{P} \exp{\big (- \eta (\ell_{f(\cdot, \hat{g})} - \ell_{f^*(\cdot, \hat{g})}) \big )} 
&= \frac{1}{\eta} \log \mathbb{E}_{P} \bigg[ \exp {\big (- \eta (\ell_{f(\cdot, \hat{g})} - \ell_{f^*(\cdot, \hat{g})}) \big )} \mathbb{1}\{\beta_{\hat{g}} ^\top  \hat{g}(X) = W \} \bigg ]  \\
&\qquad \qquad + \frac{1}{\eta} \log \mathbb{E}_{P} \bigg [ \exp {\big (- \eta (\ell_{f(\cdot, \hat{g})} - \ell_{f^*(\cdot, \hat{g})}) \big )}\mathbb{1}\{\beta_{\hat{g}} ^\top \hat{g}(X) \neq W \} \bigg ]   \\
& = \frac{1}{\eta} \log \mathbb{E}_{P} \bigg[ \exp {\big (- \eta (\ell_{f(\cdot, g_0)} - \ell_{f^*(\cdot, g_0)}) \big )} \mathbb{1}\{ \beta_{\hat{g}} ^\top \hat{g}(X) = W \} \bigg ]  \\
&\qquad \qquad + \frac{1}{\eta} \log \mathbb{E}_{P} \bigg [ \exp {\big (- \eta (\ell_{f(\cdot, \hat{g})} - \ell_{f^*(\cdot, \hat{g})}) \big )} \mathbb{1}\{ \beta_{\hat{g}} ^\top \hat{g}(X) \neq W \} \bigg ]  
\end{split}
\end{align*}

where the second line follows from the fact that for any $f$ in the event $\{ \beta_{\hat{g}} ^\top\hat{g}(X) = W \}$ we have $\ell_{f(\cdot, \hat{g})} = \ell_{f(\cdot, g_0)}$ and $\ell_{f^*(\cdot, \hat{g})} = \ell_{f^*(\cdot, g_0)}$. This is because $|\ell_{f(\cdot, \hat{g})}(X,Y) - \ell_{f(\cdot, g_0)}(X,Y) | \leq  L \ell^\text{weak} (\beta_{\hat{g}} ^\top \hat{g}(X) , \beta_{g_0}^\top g_0(X)) = L \ell^\text{weak} (W,W) = 0$. 

Dropping the indicator $\mathbb{1}\{ \beta_{\hat{g}} ^\top \hat{g}(X) = W\}$ from the integrand yields $\frac{1}{\eta} \log \mathbb{E}_{P} \big [e ^{- \eta (\ell_{f} - \ell_{f^*} )} \big ]$ which is upper bounded by $\varepsilon$ by the weak central condition. We may therefore upper bound the second term by,
\begin{align*}
\begin{split}
\frac{1}{\eta} \log \mathbb{E}_{P} \bigg [ \exp { \big (- \eta (\ell_{f(\cdot, \hat{g})} - \ell_{f^*(\cdot, \hat{g})}) \big )} \mathbb{1}\{\beta_{\hat{g}} ^\top  \hat{g}(X) \neq W \} \bigg ]  & \leq  \frac{1}{\eta} \log \mathbb{E}_{P} \bigg [\exp{(\eta B)} \mathbb{1}\{ \beta_{\hat{g}} ^\top \hat{g}(X) \neq W \} \bigg ]  \\
& \leq  \frac{\exp{(\eta B)}  }{\eta}  \mathbb{P}_{P}(\beta_{\hat{g}} ^\top \hat{g}(X) \neq W )  \\
& = \frac{\exp{(\eta B)}}{\eta} \text{Rate}(\mathcal{G}, \mathcal{D}_m^\text{weak}). \\
\end{split}
\end{align*}
The first inequality uses the fact that $\ell$ is bounded by $B$, the second line uses the basic fact $\log x \leq x$, and the final equality holds with probability $1-\delta$ by assumption. Combining this bound with the $\varepsilon$ bound on the first term yields the claimed result.
\end{proof}

\begin{prop}[Proposition \ref{prop: regression strong implies epsilon weak}]
Suppose that $\ell^\text{weak}(w,w') = \norm{w-w'}$ and that $\ell$ is bounded by $B >0$, $\mathcal{F}$ is $L$-Lipschitz relative to $\mathcal{G}$, and that $(\ell, P, \mathcal{F})$ satisfies the $\varepsilon$-weak central condition. Then $(\ell, \hat{P}, \mathcal{F})$ satisfies the $\varepsilon + \mathcal{O}\big (\sqrt{L \text{Rate}_m(\mathcal{G}, P_{X,W} )}\big )$-weak central condition with probability at least $1-\delta$. 
\end{prop}

\begin{proof}[Proof of Proposition \ref{prop: regression strong implies epsilon weak}]
For any $\delta > 0$ we can split the objective we wish to bound into two pieces as follows,
\begin{align*}
\frac{1}{\eta} \log \mathbb{E}_{\hat{P}} \exp { \big (- \eta (\ell_f - \ell_{f^*}) \big )} &= \underbrace{\frac{1}{\eta} \log \mathbb{E}_{\hat{P}} \bigg[ \exp { \big (- \eta (\ell_f - \ell_{f^*}) \big )} \mathbb{1} \bigg \{ \norm{\beta_{\hat{g}} ^\top \hat{g}(X) - W} \leq \frac{\delta}{L} \bigg \} \bigg ]}_{ =: \text{ I}}   \\
&\qquad \qquad + \underbrace{\frac{1}{\eta} \log \mathbb{E}_{\hat{P}} \bigg [ \exp {\big (- \eta (\ell_f - \ell_{f^*}) \big )}\mathbb{1} \bigg \{ \norm{\beta_{\hat{g}} ^\top \hat{g}(X) - W} > \frac{\delta}{L} \bigg \}  \bigg ] }_{ =: \text{ II}}.
\end{align*}

We will bound each term separately. The first term can be rewritten as,

\[ \text{I} = \frac{1}{\eta} \log \mathbb{E}_{P} \bigg[ \exp {\big (- \eta (\ell_{f(\cdot , \hat{g})}  - \ell_{f^*(\cdot, \hat{g})}) \big )} \mathbb{1} \bigg \{ \norm{\beta_{\hat{g}} ^\top \hat{g}(X) - W} \leq \frac{\delta}{L} \bigg \} \bigg ] \]

Let us focus for a moment specifically on the exponent, which we can break up into three parts,

\[ \ell_{f(\cdot , \hat{g})}  - \ell_{f^*(\cdot, \hat{g})}  = (\ell_{f(\cdot ,  g_0)}  - \ell_{f^*(\cdot, g_0)} ) 
+ (\ell_{f(\cdot , \hat{g})}  -  \ell_{f(\cdot, g_0)} ) 
+(\ell_{f^*(\cdot , g_0)}  - \ell_{f^*(\cdot, \hat{g})} ). \]

In the event that $ \bigg \{ \norm{\beta_{\hat{g}} ^\top \hat{g}(X) - W} \leq \frac{\delta}{L} \bigg \}$ the second and third terms can be bounded using the Lipschitzness of $\ell$, and the relative Lipschitzness of $\mathcal{F}$ with respect to $\mathcal{G}$, 

\begin{align*}
 |\ell_{f(\cdot , \hat{g})}(X,Y)  -  \ell_{f(\cdot, g_0)}(X,Y) |
+|\ell_{f^*(\cdot , g_0)}(X,Y)  - \ell_{f^*(\cdot, \hat{g})}(X,Y) | &\leq 
 L  \norm{\beta_{\hat{g}} ^\top \hat{g} - \beta_{g_0}^\top g_0} + L  \norm{\beta_{\hat{g}} ^\top \hat{g} - \beta_{g_0}^\top g_0}  \\
&=  2L \norm{\beta_{\hat{g}} ^\top \hat{g} - W} \\
&\leq 2\delta. 
\end{align*}

Plugging this upper bound into the expression for $\text{I}$, we obtain the following bound 

\begin{align*}
\text{I} &\leq \frac{1}{\eta} \log \mathbb{E}_{P} \bigg[ \exp{ \big (- \eta (\ell_{f}  - \ell_{f^*} ) \big )}\mathbb{1} \bigg \{ \norm{\beta_{\hat{g}} ^\top \hat{g}(X) - W} \leq \frac{\delta}{L} \bigg \} \bigg ] \\
&\qquad \qquad +  \frac{1}{\eta} \log \mathbb{E}_{P} \bigg[ \exp{ (2 \eta \delta )}\mathbb{1} \bigg \{ \norm{\beta_{\hat{g}} ^\top \hat{g}(X) - W} \leq \frac{\delta}{L} \bigg \} \bigg ]  \\
&\leq \frac{1}{\eta} \log \mathbb{E}_{P} \big[ \exp{ \big (- \eta (\ell_{f}  - \ell_{f^*} ) \big )} \big ] + 2\delta \\
&\leq \varepsilon + 2\delta
\end{align*}

where in the second line we have simply dropped the indicator function from both integrands, and for the third line we have appealed to the $\varepsilon$-weak central condition. 
Next we proceed to bound the second term (line by line explanations are below) $\text{II}$ by,
\vspace{-5pt}

\begin{align*}
\frac{1}{\eta} \log \mathbb{E}_{\hat{P}} \bigg [   \exp{\big (- \eta (\ell_f - \ell_{f^*}) \big )}\mathbb{1} \bigg \{ \norm{ \beta_{\hat{g}} ^\top \hat{g}(X) - W} > \frac{\delta}{L} \bigg \}  \bigg ] 
&\leq\frac{1}{\eta} \log \mathbb{E}_{\hat{P}} \bigg [ \exp{(\eta B)}\mathbb{1} \bigg \{ \norm{ \beta_{\hat{g}} ^\top \hat{g}(X) - W} > \frac{\delta}{L} \bigg \}  \bigg ] \\ 
&\leq \frac{\exp{(\eta B)}}{\eta} \mathbb{E}_{P_{X,W}} \bigg [ \mathbb{1} \bigg \{ \norm{\beta_{\hat{g}} ^\top \hat{g}(X) - W} > \frac{\delta}{L} \bigg \}  \bigg ] \\
&= \frac{\exp{(\eta B)}}{\eta} \mathbb{P}_{P_{X,W}} \bigg ( \norm{\beta_{\hat{g}} ^\top \hat{g}(X) - W} > \frac{\delta}{L} \bigg ) \\
&\leq \frac{L\exp{(\eta B)}}{\delta \eta} \mathbb{E}_{P} \norm{\beta_{\hat{g}} ^\top \hat{g}(X) - W} \\
&\leq \frac{L\exp{(\eta B)}}{\delta \eta} \text{Rate}_m(\mathcal{G},P_{X,W})
\end{align*}

where the first line follows since $\ell$ is bounded by $B$, the second line since $\log x \leq x$, the fourth line is an application of Markov's inequality, and the final inequality holds by definition of $\text{Rate}_m(\mathcal{G},P_{X,W})$ with probability $1-\delta$. Collecting these two results together we find that 

\begin{align*}
\frac{1}{\eta} \log \mathbb{E}_{\hat{P}} \exp{\big (- \eta (\ell_f - \ell_{f^*}) \big )} &= \text{I} + \text{II} \leq  \varepsilon + 2 \delta + \frac{L\exp{(\eta B)}}{\delta \eta} \text{Rate}_m(\mathcal{G}, P_{X,W}).
\end{align*}

Since this holds for any $\delta > 0$ we obtain the bound,
\vspace{-10pt}

\begin{align*}
\frac{1}{\eta} \log \mathbb{E}_{\hat{P}} \exp{\big (- \eta (\ell_f - \ell_{f^*}) \big )} &\leq \varepsilon + \min_{\delta > 0} \bigg \{2 \delta + \frac{L\exp{(\eta B)}}{\delta \eta} \text{Rate}_m(\mathcal{G}, P_{X,W})
 \bigg \} \\
&= \varepsilon + 2 \sqrt{2} \sqrt{ \frac{L \exp{(\eta B)} }{\eta}} \sqrt{\text{Rate}_m(\mathcal{G}, P_{X,W})}.
\end{align*}

The minimization is a simple convex problem that is solved by picking $\delta$ to be such that the two terms are balanced.
\end{proof}

The next proposition shows that the weak central condition is sufficient to obtain excess risk bounds. This result generalizes Theorem $1$ of \cite{mehta2016fast}, which assumes the strong central condition holds. In contrast, we make only need the weaker assumption that the weak central condition holds. \\

\begin{prop}[Proposition \ref{prop: generalization for weak central condition}]
Suppose $(\ell, Q, \mathcal{F})$ satisfies the $\varepsilon$-weak central condition, $\ell$ is bounded by $B>0$, each $\mathcal{F}$ is $L'$-Lipschitz in its parameters in the $\ell_2$ norm, $\mathcal{F}$ is contained in the Euclidean ball of radius $R$, and $\mathcal{Y}$ is compact. Then when $\text{Alg}_n(\mathcal{F},Q)$ is ERM, the excess risk $\mathbb{E}_{Q}  [\ell_{\hat{f}}(U) - \ell_{f^*}(U)  ]$ is bounded by,

\[  \mathcal{O} \bigg ( V \frac{d \log \frac{RL'}{\varepsilon}  + \log \frac{1}{\delta} }{n} + V \varepsilon \bigg ) .  \]

with probability at least $1-\delta$, where $V = B + \varepsilon$. 
\end{prop}

\begin{proof}[Proof of Proposition \ref{prop: generalization for weak central condition}] 

Before beginning the proof in earnest, let us first introduce a little notation, and explain the high level proof strategy. We use the shorthand $\Delta_f = \ell_f - \ell_{f^*}$. Throughout this proof we are interested in the underlying distribution $Q$. So, to avoid clutter, throughout the proof we shall write $\mathbb{E}$ and $\mathbb{P}$ as short hand for  $\mathbb{E}_{U \sim Q}$ and $\mathbb{P}_{U \sim Q}$. \\ 

Our strategy is as follows: we wish to determine an $a > 0$ for which, with high probability, ERM does not select a function $f \in \mathcal{F}$ such that $\mathbb{E}\Delta_f \geq  \frac{a}{n}$. 
Defining $\mathcal{F}_\beta = \{ f \in \mathcal{F} : \mathbb{E} \Delta_f \geq \beta \}$ this is equivalent to showing that, with high probability, ERM does not select a function $f \in \mathcal{F}_{\beta_n}$ where $\beta_n = \frac{a}{n}$.  In turn this can be re-expressed as showing with high probability that,
\begin{equation}\label{eqn: main thm proof}
\frac{1}{n} \sum_{j=1}^n \Delta_{f}(U_j) >  0
\end{equation}
 for all  $ f \in \mathcal{F}_{\beta_n}$, where the random variables $\{U_j\}_j$ are i.i.d samples from $Q$. In order to prove this we shall take a finite cover $\{ f_1, f_2, \ldots , f_s\}$ of our function class $\mathcal{F}_{\beta_n}$ and show that, with high probability $\frac{1}{n} \sum_{j=1}^n \Delta_{f}(U_j) >  c$ for all $f_i$ for some constant $c >0$ depending on the radius of the balls. To do this, we use the central condition, and two important tools from probability whose discussion we postpone until Appendix Section \ref{ap: prob tools},  to bound the probability of selecting each $f_i$, then apply a simple union bound. This result, combined with the fact that every element of $\mathcal{F}_{\beta_n}$ is close to some such $f_i$ allows us to derive equation (\ref{eqn: main thm proof}) for all members of the class $\mathcal{F}_{\beta_n}$.

With the strategy laid out, we are now ready to begin the proof in detail. We start by defining the required covering sets. Specifically, let $\mathcal{F}_{\beta_n, \varepsilon} $ be an optimal proper\footnote{ For a metric space $(M,\rho)$, let $S \subseteq M$. A set $E \subseteq M$  is an $\varepsilon$-cover for $S$, if for every $s \in S$ there is an $e \in E$ such that $\rho(s,e) \leq \varepsilon$. An $\varepsilon$-cover is optimal if it has minimal cardinality out of all $\varepsilon$-covers. $E$ is known as a proper cover if $E \subseteq S$.} $\varepsilon/L's$-cover of $\mathcal{F}_{\beta_n}$ in the $\ell_2$-norm, where we will pick $s$ later. It is a classical fact (see e.g. \cite{carl1990entropy} ) that the $d$-dimensional $\ell_2$-ball of radius $R$ has $\varepsilon$-covering number at most $ (\frac{4R}{\varepsilon})^d$. Since the cardinality of an optimal proper $\varepsilon$-covering number is at most the  $\varepsilon /2$-covering number, and $\mathcal{F}$ is contained in the  the $d$-dimensional $\ell_2$-ball of radius $R$, we have  $|\mathcal{F}_{\beta_n, \varepsilon}| \leq (\frac{8RL's}{\varepsilon})^d$. Furthermore, since $\ell$ is continously differentiable, $\mathcal{Y}$ is compact and $f$ is Lipschitz in its parameter vector, we have that $f \mapsto \ell_f$ is $L's$-Lipschitz in the $\ell_2$ norm in the domain and $\ell_\infty$-norm in the range (for some $s$, which we have now fixed). Therefore the proper $\varepsilon/L's$-cover of $\mathcal{F}_{\beta_n}$ pushes forward to a proper $\varepsilon$-cover of $\{ \ell_f : f \in  \mathcal{F}_{\beta_n} \}$ in the $\ell_\infty$-norm.

We now tackle the key step in the proof, which is to upper bound the probability that ERM selects an element of  $\mathcal{F}_{\beta_n, \varepsilon}$. To this end, fix an $f \in \mathcal{F}_{\beta_n, \varepsilon}$. Since  $(\ell, \hat{P}, \mathcal{F})$ satisfies the $\varepsilon$-weak central condition, we have $\mathbb{E} \big [ e^{-\eta \Delta_f} \big ] \leq e^{ \eta \varepsilon}$. Rearranging yields, 
\[  \mathbb{E} \big [ \exp{\big (-\eta ( \Delta_f+ \varepsilon) \big )} \big ] \leq 1.\]
Lemma \ref{lem: technical less than 0} implies that for any $0 <\gamma < a $ there exists a modification $\widetilde{\Delta} _f + \varepsilon$ of $\Delta_f + \varepsilon$, and an $\eta \leq \eta_f \leq 2\eta$ such that  $\widetilde{\Delta} _f \leq \Delta_f$, almost surely, and,

\begin{equation}
 \mathbb{E} \big [ \exp{\big (-\eta_f ( \widetilde{\Delta} _f+ \varepsilon)\big )} \big ] = 1
   \quad\mathrm{and}\quad 
 \mathbb{E} \widetilde{\Delta} _f \geq \frac{a-\gamma}{n}.
\end{equation}

Since $\widetilde{\Delta} _f+ \varepsilon$ belongs to the shifted interval $[-V,V]$ where $V = B + \varepsilon$, Corollaries $7.4$ and $7.5$ from \cite{van2015fast} imply\footnote{Note that although the Corollaries in \cite{van2015fast} are stated specifically for $\Delta_f$, the claims hold for \emph{any} random variable satisfying the hypotheses, including our case of $\Delta_f + \varepsilon$.} that,

\[ \log  \mathbb{E} \big [ \exp{\big (-\eta_f/2 ( \widetilde{\Delta} _f+ \varepsilon) \big )} \big ]  \leq - \frac{0.18  }{( V  \vee 1/ \eta_f  )} \bigg ( \frac{a - \gamma }{n} + \varepsilon \bigg )  \leq - \frac{0.18(a-\gamma)  }{( V  \vee 1/ \eta_f  )n} . \]

where we define $a' = a - \gamma$. By Cram\'er-Chernoff (Lemma \ref{lem: Cramer-Chernoff})  with $t = c a' \varepsilon$ (where $c$ will also be chosen later) and the $\eta$ in the lemma being $\eta_f/2$, we obtain
\begin{align*} \mathbb{P} \bigg ( \frac{1}{n} \sum_{j=1}^n \big ( \widetilde{\Delta} _f(U_j) + \varepsilon \big )\leq c a' \varepsilon \bigg )  
&\leq \exp \bigg (  -\frac{0.18}{V  \vee 1 / \eta_f}a' + \frac{n\eta_f c a' \varepsilon}{2} \bigg ) \\
&\leq \exp \bigg (  -\frac{0.18}{V  \vee 1 / \eta}a'+ n\eta c a' \varepsilon \bigg ) \\
&=  \exp  (-Ca' )
\end{align*}
where $C :=  \frac{0.18}{B  \vee 1/\eta} -  n\eta c  \varepsilon $, and the second inequality follows since $\eta \leq \eta_f \leq 2\eta$. Let us now pick $c$ so as to make $C$ bigger than zero, and in particular so that $C =  \frac{0.09}{B \vee 1/\eta}$. That is, let $c =  \frac{1}{n \varepsilon}\frac{0.09}{V \eta \vee 1} $. Using the fact that $a' - 2/c \leq a'$, and a union bound over $f \in \mathcal{F}_{\beta_n, \varepsilon}$ we obtain a probability bound on all of  $\mathcal{F}_{\beta_n, \varepsilon}$,

\[ \mathbb{P} \bigg ( \exists f \in \mathcal{F}_{\beta_n, \varepsilon} : \frac{1}{n} \sum_{j=1}^n \widetilde{\Delta}_f(U_j) \leq (ca'-1)\varepsilon \bigg ) \leq \bigg (\frac{8RL's}{\varepsilon} \bigg )^d \exp  \bigg ( - \frac{0.09}{B \vee \frac{1}{\eta}} (a' -2/c) \bigg ). \]

Define the right hand side to equal $0< \delta < 1$. Note that we are allowed to do this thanks to the fact $C > 0$, which implies that the right hand side goes to zero as $a' \rightarrow \infty$ . This makes it possible to pick a sufficiently large $a'$ for which the right hand side is less than $1$. Solving for $a = a' + \gamma $ we choose,
\[ a = \frac{V  \vee 1/ \eta}{0.09} \bigg ( d \log \frac{8RL's}{\varepsilon} + \log \frac{1}{\delta} \bigg ) + 2/c + \gamma.  \]

Therefore, with probability at least $1-\delta$ we have for all $ f \in \mathcal{F}_{\beta_n, \varepsilon}$ that $ \frac{1}{n} \sum_{j=1}^n \widetilde{\Delta}_f(U_j) > (ca'-1)\varepsilon$. 
Therefore, for any  $ f' \in \mathcal{F}_{\beta_n}$ we can find $ f \in \mathcal{F}_{\beta_n, \varepsilon}$  such that $\Vert \ell_f-\ell_{f'} \Vert_\infty \leq \varepsilon$. 

Finally, since  $ca \geq 2$ for sufficiently small $\varepsilon$ by construction, and $ \Delta_{f} \geq \widetilde{\Delta}_f$ almost surely, we find that $ \frac{1}{n} \sum_{j=1}^n \Delta_{f'}(U_j)  \geq  \frac{1}{n} \sum_{j=1}^n \Delta_{f}(U_j) - \varepsilon \geq  \frac{1}{n} \sum_{j=1}^n \widetilde{\Delta}_f(U_j)  - \varepsilon \geq  (ca-1)\varepsilon - \varepsilon > 0$. 
We have proven that with probability at least $1- \delta$ that $\frac{1}{n} \sum_{j=1}^n \Delta_{f'}(U_j) >  0$ for all  $ f' \in \mathcal{F}_{\beta_n}$. Therefore, with high probability, ERM will not select any element of  $\mathcal{F}_{\beta_n}$. Finally, the bound described in the theorem comes from substituting in the choice of $c$, and rounding up the numerical constants, recognizing that since the claim holds for all $\gamma > 0$ , we may take the limit as $\gamma \rightarrow 0^+$ to obtain,
\[ a \leq  12(V  \vee 1/ \eta)\bigg ( d \log \frac{8RL's}{\varepsilon}  + \log \frac{1}{\delta} \bigg ) + 12(V \eta \vee 1) n \varepsilon + 1.  \]
\end{proof} 

The heavy lifting has now been done by the previous propositions and theorems. In order to obtain the main result, all that remains now is to apply each result in sequence. \\

\begin{theorem}[Theorem \ref{thm: fast rates}]
Suppose that $(\ell, P, \mathcal{F} )$ satisfies the central condition and that $\text{Rate}_m(\mathcal{G}, P_{X,W} ) = \mathcal{O}(1 / m^\alpha)$. Then  when $\text{Alg}_n(\mathcal{F},\hat{P})$ is ERM we obtain excess risk $\mathbb{E} _P [\ell_{\hat{h}}(X,Y) - \ell_{h^*}(X,Y) ]$ that is bounded by,

\begin{equation*}
\mathcal{O} \bigg (  \frac{  d \alpha \beta \log RL'n + \log \frac{1}{\delta} }{n} + \frac{ L  }{n^{\alpha\beta}}   \bigg )    
\end{equation*}
\vspace{-3pt}
with probability at least $1-\delta$, if either of the following conditions hold,
\vspace{-3pt}
\begin{enumerate}
	\item  $m = \Omega(n^{\beta})$ and  $\ell^\text{weak}(w,w') = \mathbb{1}\{ w \neq w'\}$ (discrete $\mathcal{W}$-space). 
	\item $m = \Omega(n^{2\beta})$  and $\ell^\text{weak}(w,w') = \norm{w-w'}$ (continuous $\mathcal{W}$-space).
\end{enumerate}
\vspace{-5pt}
\end{theorem}

\begin{proof}[Proof of Theorem \ref{thm: fast rates}]
\textbf{Case 1:}
We have $m = \Omega(n^{\beta})$, and $\text{Rate}_m(\mathcal{G}, P_{X,W}) = \mathcal{O}(1 / m^\alpha)$, together impling that  $\text{Rate}(\mathcal{G}, \mathcal{D}_m^\text{weak}) = \mathcal{O}(1 / n^{\alpha \beta})$. We apply Proposition \ref{prop: categorical strong implies epsilon weak} to conclude that $(\ell, \hat{P}, \mathcal{F})$ satisfies the $ \mathcal{O}(1 / n^{\alpha \beta})$-weak central condition with probability at least $1-\delta$. 

Proposition \ref{prop: generalization for weak central condition} therefore implies that $\text{Rate}_n(\mathcal{F}, \hat{P}) = \mathcal{O} \bigg ( \frac{d \alpha \beta \log 8RL'n  + \log \frac{1}{\delta} }{n} + \frac{1}{n^{\alpha \beta} }\bigg )$.

Combining these two bounds using Proposition \ref{thm: general bound} we conclude that 

\[  \mathbb{E}  [\ell_{\hat{h}}(Z) - \ell_{h^*}(Z) ] \leq \mathcal{O} \bigg ( \frac{d \alpha \beta \log 8RL'n  + \log \frac{1}{\delta} }{n} + \frac{L }{n^{\alpha \beta} }\bigg ). \]

\textbf{Case 2:}
The second case is proved almost identically, however note that since in this case we have $m = \Omega(n^{2\beta})$, that now $\text{Rate}_m(\mathcal{G}, P_{X,W}) = \mathcal{O}(1 / n^{2\alpha \beta})$. The factor of two is cancelled our by the extra square root factor in Proposition \ref{prop: regression strong implies epsilon weak}. The rest of the proof is exactly the same as case $1$.

\end{proof}

\section{Probabilistic Tools}\label{ap: prob tools}

In this section we present two technical lemmas that are key tools used to prove Proposition \ref{prop: generalization for weak central condition}. The first allows us to take a random variable $\Delta$ such that $\mathbb{E}e^{-\eta \Delta } \leq 1$ and perturb downwards it slightly to some $\widetilde{\Delta} \leq \Delta $ so that the inequality becomes an equality (for a slightly different $\eta$) and yet the expected value changes by an arbitrarily small amount.   \\

\begin{lemma}\label{lem: technical less than 0}
Suppose $\eta > 0$ and $\Delta$ is an absolutely continuous random variable on the probability space $( \Omega , \mathbb{P})$ such that $\Delta$ is almost surely bounded, and $\mathbb{E}e^{-\eta \Delta } \leq 1$. Then for any $\varepsilon > 0$ there exists an $\eta \leq \eta ' \leq 2\eta$ and another random variable $\widetilde{\Delta}$ (called a ``modification'') such that,

\begin{enumerate}
	\item $\widetilde{\Delta} \leq \Delta$ almost surely,
	\item $\mathbb{E}e^{-\eta ' \widetilde{\Delta} } = 1$, and
	\item $| \mathbb{E}[\Delta -\widetilde{\Delta} ] | \leq \varepsilon$.
\end{enumerate}  
\end{lemma}

\begin{proof}
We may assume that  $\mathbb{E}e^{-\eta \Delta } < 1$ since otherwise we can simply take $\widetilde{\Delta} = \Delta$ and $\eta = \eta ' $. Due to absolute continuity, for any  $\delta > 0$ there is a measurable set $A_\delta \subset \Omega$ such that $\mathbb{P}(A_\delta) = e^{- 1/\delta}$. Now define $\widetilde{\Delta} : \Omega \rightarrow \mathbb{R}$ by, 
\vspace{-10pt}

 \begin{equation}
    \widetilde{\Delta}(\omega) =
    \begin{cases*}
      \Delta(\omega) & if $\omega \notin A_\delta$ \\
      - \frac{1}{2 \delta \eta}        & if $\omega \in A_\delta$
    \end{cases*}
  \end{equation}
  
  We now prove that as long as $\delta$ is small enough, all three claimed properties hold. \\
  
 \textbf{Property 1:} 
 Since $\Delta$ is almost surely bounded, there is a $V>0$ such that $|\Delta| \leq V $ almost surely. Taking $\delta$ small enough that $- \frac{1}{2 \delta \eta}  \leq  - V$ we guarantee that $\widetilde{\Delta} \leq \Delta$ almost surely.
 
  \textbf{Property 2:} 
  We can lower bound the $2\eta$ case, 
  \[\mathbb{E}e^{-2 \eta \widetilde{\Delta} } \geq e^{- 2 \eta  ( - \frac{1}{ 2 \eta \delta}  )} \mathbb{P}(A _\delta)  
  = e^{1/\delta} \mathbb{P}(A _\delta) 
 =  e^{1/\delta}  e^{-1/\delta} 
  = 1.\]
 We can similarly upper bound the $\eta$ case,
 \begin{align*}
  \mathbb{E}e^{- \eta \widetilde{\Delta} } &=  \int e^{-\eta \widetilde{\Delta}(\omega)} \mathbf{1} \{ \omega \in A_\delta \} \mathbb{P} (\text{d} \omega) + \int e^{-\eta \widetilde{\Delta}(\omega)} \mathbf{1} \{ \omega \notin A_\delta \} \mathbb{P} (\text{d} \omega) \\
  &= e^{1/2\delta} \mathbb{P}(A _\delta)  + \int e^{-\eta \Delta(\omega)} \mathbf{1} \{ \omega \notin A_\delta \} \mathbb{P} (\text{d} \omega) \\
  &\leq e^{-1/2\delta} +   \int e^{-\eta \Delta(\omega)} \mathbb{P} (\text{d} \omega) \\
   &\leq e^{-1/2\delta}   +   \mathbb{E} e^{-\eta \Delta}.
 \end{align*}
  Recall that by assumption $ \mathbb{E} e^{-\eta \Delta} < 1$, so we may pick $\delta$ sufficiently small so that $e^{-1/2\delta}   +   \mathbb{E} e^{-\eta \Delta}   < 1$. Using these two bounds, and observing that boundedness of $\Delta$ implies continuity of $\eta \mapsto   \mathbb{E} \big [ e^{-\eta \Delta } \big ] $, we can guarantee that there is an  $\eta \leq \eta ' \leq 2 \eta$ such that $\mathbb{E} \big [ e^{-\eta ' \widetilde{\Delta} } \big ] = 1$.  
\vspace{-3pt}

  \textbf{Property 3:} 
  Since $\Delta$ and $\widetilde{\Delta}$ only disagree on $A_\delta$, 
  \[
  \mathbb{E} | \widetilde{\Delta} - \Delta  |  
 = \int |\widetilde{\Delta}(\omega) - \Delta(\omega) |  \mathbf{1}\{ w \in A_\delta \} \mathbb{P}(\text{d}\omega) 
 \leq \bigg (\frac{1}{2\delta \eta} + V \bigg ) \mathbb{P}(A_\delta)
 = \bigg (\frac{1}{2\delta \eta} + V \bigg ) e^{-1/\delta} 
\]
 
 which converges to $0$ as $\delta \rightarrow 0^+$. We may, therefore, make the difference in expectations smaller than $\varepsilon$ by taking $\delta$ to be sufficiently close to $0$.
\end{proof}
\vspace{-10pt}

The second lemma is a well known Cram\'er-Chernoff bound that is used to obtain concentration of measure results. A proof was given, for example, given in \cite{van2015fast}. However, since the proof is short and simple we include it here for completeness. \\

\begin{lemma}[Cram\'er-Chernoff  \cite{van2015fast} ]\label{lem: Cramer-Chernoff}
Let $\Delta, \Delta_1, \ldots , \Delta_n$ be i.i.d. and define $\Lambda_\Delta(\eta) = \log \mathbb{E}[ e^{- \eta \Delta}]$. Then, for any $\eta >0$ and $t \in \mathbb{R}$, 

\[  \mathbb{P} \bigg ( \frac{1}{n} \sum_{i=1}^n \Delta_i \leq t \bigg ) \leq \exp \bigg ( \eta n t + n \Lambda_\Delta(\eta) \bigg ). \]
\vspace{-10pt}
\end{lemma}
\vspace{-10pt}

\begin{proof}
Note that since $x \mapsto \exp(-\eta x)$ is a bijection, we have,
 \[ \mathbb{P} \bigg ( \frac{1}{n} \sum_{i=1}^n \Delta_i \leq t \bigg ) =  \mathbb{P} \bigg (\exp \big ( - \eta \sum_{i=1}^n \Delta_i \big ) \geq \exp( - \eta n t) \bigg ).\] Applying Markov's inequality to the right hand side of the equality yields the upper bound,
 \[ \exp(\eta n t) \mathbb{E} \big [ \exp ( - \eta\sum_{i=1}^n \Delta_i  ) \big ]   = \exp(\eta n t) \big [ \mathbb{E} \  \exp ( - \eta \Delta  ) \big  ]^n =   \exp \bigg ( \eta n t + n \Lambda_\Delta(\eta) \bigg ).\]
\end{proof}

\section{Hyperparameter and Architecture Details}

All models were trained using PyTorch \cite{paszke2019pytorch} and repeated from scratch $4$ times to give error bars. All layers were initialized using the default uniform initialization. 

\paragraph{Architecture} For the MNIST experiments we used the ResNet-$18$ architecture as a deep feature extractor for the weak task \cite{he2016deep}, followed by a single fully connected layer to the output. For the strong model, we used a two hidden layer fully connected neural network as a feature extractor with ReLU activations. The first hidden layer has $2048$ neurons, and the second layer has $1024$. This feature vector is then concatenated with the ResNet feature extractor, and passed through a fully connected one hidden layer network with $1024$ hidden neurons. For all other datasets (SVHN, CIFAR-$10$, CIFAR-$100$) the exact same architecture was used except for replacing the ResNet-$18$ feature extractor by ResNet-$34$. We also ran experiments using smaller models for the weak feature map, and obtained similar results. That is, the precise absolute learning rates changed, but the comparison between the learning rates remained the similar.  

\paragraph{Optimization} We used Adam \cite{kingma2014adam} with initial learning rate $0.0001$,  and $\beta_1 = 0.5$, and $\beta_2 = 0.999$. We used batches of size  $100$, except for MNIST, for which we used $50$. We used an exponential learning rate schedule, scaling the learning rate by $0.97$ once every two epochs.

\paragraph{Data pre-processing} For CIFAR-$10$, CIFAR-$100$, and SVHN we used random cropping and horizontal image flipping to augment the training data. We normalized CIFAR-$100$ color channels by subtracting the dataset mean pixel values $(0.5071, 0.4867, 0.4408)$ and dividing by the standard deviation $(0.2675, 0.2565, 0.2761)$. For CIFAR-$10$ and SVHN we normalize each pixel to live in the interval $[-1,1]$ by channel-wise subtracting $(0.5,0.5,0.5)$ and dividing by $(0.5,0.5,0.5)$. For MNIST the only image processing was to normalize each pixel to the range $[0,1]$. 

\paragraph{Number of training epochs} The weak networks were trained for a number of epochs proportional to $1/m$. For example, for all CIFAR-$10$ experiments the weak networks were trained for $500000/m$ epochs. This was sufficient to train all models to convergence.

Once the weak network was finished training, we stopped all gradients passing through that module, thereby keeping the weak network weights fixed during strong network training. To train the strong network, we used early stopping to avoid overfitting. Specifically, we tested model accuracy on a holdout dataset once every $5$ epochs. The first time the accuracy decreased we stopped training, and measured the final model accuracy using a test dataset.

\paragraph{Dataset size} The amount of strong data is clearly labeled on the figures. For the weak data, we used the following method to compute the amount of weak data to use:

 \vspace{-9pt}
\begin{align*}
m^{(1)}_i = c_1n_i \\ 
m^{(2)}_i = c_2 n_i^2 
\end{align*}
 \vspace{-9pt}

where $m^{(1)}_i$ is the amount of weak data for the linear growth, $m^{(2)}_i$  for quadratic growth, and $n_1, n_2, \ldots , n_7$ are the different strong data amounts. For MNIST we took $(c_1,c_2) = (4,0.02)$, for SVHN we took $(c_1,c_2) = (4.8,0.0024)$ and for CIFAR-$10$ and CIFAR-$100$ we took $(c_1,c_2) = (4,0.002)$. An important property in each case is that $m^{(1)}_1 = m^{(2)}_1$, i.e. weak and quadratic growth begin with the same amount of weak labels. 

\label{ap: experiments}

\end{document}